\let\ORIlabel\label
\let\ORIrefstepcounter\refstepcounter
\AddToHook{package/hyperref/before}{
   \let\label\ORIlabel 
   \let\refstepcounter\ORIrefstepcounter}
   
\documentclass{siamart190516}

\usepackage[utf8]{inputenc}

\usepackage{amsmath,amssymb,amsfonts,color,multirow,comment,pgfplots,pgfplotstable,tikz,cleveref,hyperref}

\usepackage{longtable}
\usepackage{booktabs}
\usepackage{graphicx}
\usetikzlibrary{plotmarks}
\usetikzlibrary{positioning}
\usetikzlibrary{arrows.meta}
\usepgfplotslibrary{patchplots}
\usepgfplotslibrary{groupplots}
\usepackage{grffile}
\usepackage{url}
\usepackage{algorithm,algpseudocode}
\usepackage{color}
\usepackage{mdframed}

\usepackage[mathlines]{lineno}

\pgfplotsset{compat=1.17}

\definecolor{greeng}{rgb}{0.0, 0.5, 0.4}

\newcommand{\name}{sCD}

\definecolor{brightpink}{rgb}{1.0, 0.0, 0.5}
\newcommand{\ngc}[1]{{\color{brightpink} (\textbf{NG:} #1)}}

\newcommand{\revise}[1]{{{\color{black} #1}}}

\definecolor{forestgreen(web)}{rgb}{0.13, 0.55, 0.13}

\DeclareMathOperator{\supp}{supp}
\DeclareMathOperator{\nnz}{nnz}
\DeclareMathOperator{\rank}{rank}
\DeclareMathOperator{\argmin}{argmin}
\DeclareMathOperator{\argmax}{argmax}
\DeclareMathOperator{\var}{var}

\title{Nonnegative Matrix Factorization in the \\  Component-Wise L1  Norm for Sparse Data\thanks{GS and NG acknowledge the support by the European Union (ERC consolidator, eLinoR, no 101085607). The work of GS\ was partially supported by INdAM-GNCS through Progetti di Ricerca. Kévin Dubrulle is a FRIA grantee of the Fonds de la Recherche Scientifique - FNRS.}}
\author{Giovanni Seraghiti\thanks{Corresponding author. University of Mons, Rue de Houdain 9, 7000 Mons, Belgium, and 
 Dipartimento di Ingegneria Industriale, Universit\`a degli Studi di Firenze, Viale Morgagni 40/44, 50134 Firenze, Italia. Member of the INdAM Research Group GNCS.}   
 \and
Kévin Dubrulle\thanks{University of Mons, Rue de Houdain 9, 7000 Mons, Belgium.} 
\and
Arnaud Vandaele\footnotemark[3]
\and
Nicolas Gillis\footnotemark[3]
}
\date{March 2021}
\begin{document}
\maketitle

\begin{abstract}
   Nonnegative matrix factorization (NMF) approximates a nonnegative matrix, $X$, by the product of two nonnegative factors, $WH$, where $W$ has $r$ columns and $H$ has $r$ rows. In this paper, we consider NMF using the component-wise L1 norm as the error measure (L1-NMF), which is suited for data corrupted by heavy-tailed noise, such as Laplace noise or salt and pepper noise, or in the presence of outliers. Our first contribution is an NP-hardness proof for L1-NMF, even when $r=1$, in contrast to the standard NMF that uses least squares. Our second contribution is to \revise{analyze, under simplified probabilistic assumptions, how the sparsity in the data enforces zero solution in the optimal scalar update in the factors of L1-NMF when all the other entries are kept fixed. This provides an intuition of the connection between the sparsity of the L1-NMF factors with the sparsity of the input. Even though sparsity favors interpretability,} if the data is affected by false zeros, too sparse solutions might degrade the model. Our third contribution is a new, more general, L1-NMF model for sparse data, dubbed weighted L1-NMF (wL1-NMF), where the sparsity of the factorization is controlled by adding a penalization parameter to the entries of $WH$ associated with zeros in the data. The fourth contribution is a new coordinate descent (CD) approach for wL1-NMF, denoted as sparse CD (sCD), where each subproblem is solved by a weighted median algorithm. \revise{Although it lacks convergence guarantees to a stationary point,} sCD is, to the best of our knowledge,  the first algorithm for L1-NMF whose complexity scales with the number of nonzero entries in the data, making it efficient in handling large-scale, sparse data. 
   We perform extensive numerical experiments on synthetic and real-world data, \revise{including imaging mass spectrometry and topic modeling}, to show the effectiveness of our new proposed model (wL1-NMF) and algorithm (sCD).
\end{abstract}



\section{Introduction}
Given a nonnegative matrix $X \in \mathbb{R}_+^{m \times n}$ and a factorization rank $r$, 
the goal of nonnegative matrix factorization (NMF) is to find $W \in \mathbb{R}_+^{m \times r}$ and $H \in \mathbb{R}_+^{r\times n}$ such that $X \approx WH$. 
NMF can be formulated as follows
\begin{equation}
    \min_{W \geq 0, \ H\geq 0} \; \; d(X, WH),
    \label{eq:gen_NMF}
\end{equation}
where $d: \mathbb{R}_+^{m \times n} \times \mathbb{R}_+^{m \times n} \to \mathbb{R}_+$ is an error measure. 
Two popular examples~\cite{lee1999learning} are (1)~Frobenius NMF (FroNMF), which is the maximum likelihood estimator (MLE) for additive Gaussian noise,  that minimizes the least squares error, $d(X,WH) = \lVert X - WH \rVert_F^2 = \sum_{i=1}^m \sum_{j=1}^n(X-WH)_{i,j}^2$,  
and (2)~Kullback-Leibler NMF (KL-NMF), which is the MLE when $X$ follows a Poisson distribution of parameters $WH$, 
and the error measure is the Kullback-Leibler divergence, $d(X,WH)=\sum_{i=1}^m \sum_{j=1}^nX_{i,j}\log \left( \frac{X_{i,j}}{(WH)_{i,j}}\right)-X_{i,j}+(WH)_{i,j}$. 
Although FroNMF, KL-NMF, and their extensions have been widely studied, see, e.g.,~\cite{CZA09, gillis2020nonnegative}, 
it is well known that they are sensitive to outliers. Moreover, FroNMF and KL-NMF exhibit poor performance when data are affected by long-tailed noise such as Laplacian noise or salt and pepper noise. To overcome these limitations, more robust NMF models have been proposed. A popular choice for robust NMF is the L21 norm~\cite{kong2011robust}, that is, 
$d(X, WH)=\lVert X- WH \rVert_{2,1}=\sum_{j=1}^n \sqrt{ \sum_{i=1}^m (X-WH)^2_{i,j}}$. 
Among robust NMF models, much less attention has been devoted to L1-NMF, which models additive Laplace noise, where the error measure is the component-wise L1 norm:
\begin{equation} 
d(X, WH)= ||X-WH||_1 =\sum_{i=1}^m \sum_{j=1}^n |X-WH|_{i,j}. 
\label{eq:l1NMF}
\end{equation}
This is due to the fact that solving L1-NMF requires more involved optimization strategies because its objective is not differentiable.  
In fact, if we remove the nonnegativity constraints in~\eqref{eq:l1NMF}, the low-rank matrix approximation (LRMA) problem in the L1 norm is known to be NP-hard, even  for $r=1$~\cite{gillis2018complexity}. \revise{Nevertheless, the use of an L1 reconstruction term is a natural candidate for improving robustness to sparse or heavy-tailed noise and outliers, and previous works have shown that L1-NMF is a promising alternative to more standard NMF formulations in these settings~\cite{guan2012mahnmf,cao2023manhattan,rahimi2024projected}.}

\paragraph{Outline and Contributions} 

In Section~\ref{sec:motiv}, we further motivate L1-NMF, and briefly review the most relevant works on the topic. 
In Section~\ref{sec:complexity}, we prove the NP-hardness of rank-one L1-NMF, that is \eqref{eq:l1NMF} when $r=1$. 
In Section~\ref{sec:sparsity_inducing}, \revise{we consider the subproblem corresponding to optimizing a single entry of the factors in L1-NMF, which is a scalar L1 regression subproblem, also known as least absolute deviation (LAD). We analyze, under simplified probabilistic assumptions, how the sparsity of the inputs affects the probability of the LAD solution to be equal to zero.} 
\revise{In  Section~\ref{sec:w_ell_1_nmf} we introduce   wL1-NMF, which is} a generalization of L1-NMF, in cases where the input data is too sparse due, for example, to the presence of false zeros. 
In Section~\ref{sec:l1_for_sparse}, we present our new algorithm, \name, for wL1-NMF.   
We provide extensive numerical experiments in Section~\ref{sec:num_res}, comparing our new model wL1-NMF with other NMF models and our new algorithm, \name, with the state of the art algorithms for L1-NMF.


Let us summarize our four main contributions: 

\begin{enumerate}
    \item We show that L1-NMF is NP-hard when $r=1$ (Theorem~\ref{lem:complexityyesyes}).

    \item 
    \revise{Assuming that the inputs are Bernoulli distributed, we show that the sparsity of the data influences the probability of having a zero optimal solution of the scalar LAD subproblem, corresponding to the optimal update of one entry of the factors of L1-NMF; see Section~\ref{sec:sparsity_inducing}. Although this is a simplified setting, it provides insight into the relation between the sparsity of the data and the sparsity of coordinate-wise minimizers\footnote{\revise{Let $\mathcal{X} = \mathcal{X}_1 \times\dots \times \mathcal{X}_n \subseteq \mathbb{R}^{n} $ and  $f : \mathcal{X} \to \mathbb{R}$.
    A point $x^* = (x_1^*, \dots, x_n^*) \in  \mathcal{X}$ is a coordinate-wise minimizer of $f$ on the domain $\mathcal X$ if for every $i=1,\dots,n$, 
    $x_i^* \in \argmin_{x_{i} \in \mathcal X_i} f(x_1^*, \dots,x^*_{i-1},x_i,x^*_{i+1},\dots,x^*_n)$.}} 
    of L1-NMF. 
    This allows us to explain our numerical findings: the factors $W$ and $H$ at a coordinate-wise minimizer (approximated by the sCD solution in our case) are sparser if the data is  sparser; see Section~\ref{sec:num_res}.}
    \item  Even though sparsity might favor the interpretability of the factorization, excessively sparse solutions might be inappropriate,  
which may happen for zero-inflated data, such as word-document count matrices. For this reason, we propose a generalization of L1-NMF, referred to as weighted L1-NMF (wL1-NMF), where a parameter $\lambda$ is introduced to downweight the importance of zeros in the input data. 
wL1-NMF uses the following error measure
\begin{equation}
    \ell(X,W,H,\lambda) := \sum_{(i,j) \in \kappa^+} \lvert X -WH \rvert_{i,j} + \lambda \sum_{(i,j) \in \kappa^0} \lvert WH \rvert_{i,j},
    \label{eq:weight_l1}
\end{equation}
where $\lambda \in [0,1]$, 
    $\kappa^+=\left\{ (i,j) \ | \ X_{i,j}>0 \right\}$ 
    and 
    $\kappa^0=\left\{ (i,j) \ | \ X_{i,j}=0 \right\}$.
wL1-NMF covers L1-NMF as a special case when $\lambda=1$, while the case $\lambda=0$ corresponds to a nonnegative low-rank matrix completion problem where zero entries correspond to missing entries. 

    \item wL1-NMF can be computed using a coordinate descent algorithm that solves one scalar LAD subproblem, for each entry of $W$ and $H$ in an alternating fashion, with a computational cost per iteration of $\mathcal{O}(mnr\log(mn))$. We propose a new variant, dubbed \name, that exploits the nonnegativity and the sparsity of the original matrix to reduce the computational cost to $\mathcal{O}(r\nnz(X)\log(\nnz(X)))$, where $\nnz(X)$ denotes the number of nonzero entries of $X$.    
\end{enumerate}

\section{\revise{L1-NMF: previous works}} 
\label{sec:motiv} 

NMF finds applications in a variety of fields, such as feature extraction~\cite{lee1999learning,guillamet2002non}, topic modeling and document classification~\cite{shahnaz2006document,ding2008equivalence}, hyperspectral unmixing~\cite{bioucas2012hyperspectral,ma2013signal}; see~\cite{CZA09, gillis2020nonnegative} for books on the topic. 
An appropriate choice of the error measure in~\eqref{eq:gen_NMF} is crucial, and it depends on the type of data and the task at hand. 

It is well known that the L1 norm is much more robust to outliers than the L2 norm in the context of regression problems~\cite{portnoy1997gaussian,chen2008analysis,ke2005robust}, or in low-rank matrix approximation (LRMA), where L1-LRMA and robust PCA are much more robust than  PCA~\cite{croux1998robust,ke2005robust,kwak2008principal,markopoulos2013some,markopoulos2014optimal}. However, the L1 norm yields optimization problems that are more difficult to solve than those involving the L2 norm. Therefore, LRMA using the component-wise L1 norm (L1-LRMA) is a relatively unexplored area. Let us mention important works addressing L1-LRMA: the alternating  optimization approach by Ke and Kanade~\cite{ke2003robust,ke2005robust}, 
the Wiberg algorithm by Eriksson and Van Den Hengel~\cite{eriksson2010efficient}, and the augmented Lagrangian approaches by Zheng et al.~\cite{zheng2012practical}. Another relevant contribution is the work of Song et al.~\cite{song2017low} who introduced an algorithm for L1-LRMA with provable approximation guarantees that scales with the number of nonzero entries in the data using sketching (that is, they use random projections of the data). 

Similarly, few NMF models consider the L1 norm directly, and robust NMF usually includes other error measures that are easier to optimize, such as the L21 norm~\cite{kong2011robust,huang2014robust}, the more general L2$p$ norm ($p>0$)~\cite{li2017robust}, or functions that smoothly interpolate between the Frobenius norm and the L1-norm, for example, the Huber function~\cite{du2012robust,guo2021modified}. Another strategy is to approximate $X$ by $WH-S$, where $S$ is a sparse matrix intended to capture the sparse additive noise in the data~\cite{zhang2011robust,shen2014robust}. 

Among the notable contributions on L1-NMF, Rahimi et al.~\cite{rahimi2024projected} showed that L1-NMF, despite being neither smooth nor convex, is paraconvex, which is a generalization of the class of weakly convex functions. Moreover, they proposed a convergent projected subgradient method for paraconvex functions, optimizing both $W$ and $H$ simultaneously, which was successfully tested on L1-NMF. Despite having theoretical guarantees, the subgradient approach might be slow in practical applications and does not exploit the property of L1-NMF to be convex with respect to each variable individually. 
Thus, block coordinate descent (BCD) is a valid alternative to compute an approximate minimizer of~\eqref{eq:l1NMF} by sequentially optimizing over $W$ and $H$ and exploiting the convexity of the subproblems, as done in~\cite{ke2003robust,ke2005robust} in the unconstrained case. 
Moreover, the subproblem for $H$ (or $W$) is separable by columns (or rows), and equivalent to $n$ (or $m$) multivariate L1 regression problems, a.k.a.\ LAD subproblems. 
Ke and Kanade~\cite{ke2003robust} showed that convex programming can be used to solve each LAD subproblem, but it yields algorithms that are not scalable for large dimensions~\cite{ke2005robust}. To address this limitation, Guan et al.~\cite{guan2012mahnmf} applied Nesterov smoothing to the nonsmooth LAD subproblem. This strategy consists of solving a smooth approximation of the dual problem and projecting the solution back to the primal space. The obtained solution is an approximate stationary point of the primal function~\cite{nesterov2005smooth}.
Another approach is to split each multivariate LAD problem into $r$ smaller one-dimensional LAD subproblems, one for each component \revise{of the factors}, as suggested in~\cite{ke2003robust,guan2012mahnmf,gillis2011dimensionality}. This is a CD method for L1-NMF.

\section{Computational complexity of L1-NMF} 
\label{sec:complexity}

FroNMF is NP-hard when $r$ is part of the input~\cite{vavasis2010complexity}, but it can be solved in polynomial time when $r=1$, via the SVD, by combining Eckart-Young and Perron-Frobenius theorems.
Unfortunately, L1-NMF is significantly more complex in this case: we show in this section that L1-NMF is NP-hard even when $r=1$. 

The computational complexity of the best rank-one  approximation in the L1 norm was investigated in \cite{gillis2018complexity}, where it was proved that the unconstrained rank-one problem,  
\begin{equation}\label{eq:l1lrageneral}
\min_{u\in\mathbb{R}^s,v\in\mathbb{R}^t}\ \|M-uv^\top\|_1, 
\end{equation} 
is NP-hard. 
The proof essentially follows a two-step approach. 
It first shows that for any matrix $M\in\{\pm1\}^{s\times t}$, the problem~\eqref{eq:l1lrageneral} admits an optimal solution $(u^*,v^*)$ such that $u^*\in\{\pm1\}^s$ and $v^*\in\{\pm1\}^t$.  
This property is then used to encode an NP-hard problem, namely \textsc{Max-Cut}. 

Unfortunately, the first argument does not directly extend to the nonnegative setting, where the input matrix is nonnegative and the factors are constrained to be nonnegative. 
In fact, consider the binary matrix $X\in\{0,1\}^{4\times 5}$ given by  
\[
X= 
\begin{pmatrix}
1&1&0&1&0\\
0&1&0&1&1\\
0&1&1&1&0\\
1&0&1&1&1
\end{pmatrix}, 
\]
and the corresponding rank-one L1-NMF problem
\[
\min_{w\geq 0,h\geq 0}\ \|X-wh^\top\|_1.
\]
An optimal binary solution can be found by enumeration, 
\[
w^* = (1,1,1,1)^\top,\quad h^*=(0,1,0,1,0)^\top, \quad \text{ with }\|X-w^* {h^*}^\top\|_1 = 7.
\]
However, there exists a nonnegative solution with a smaller error:
$$
w'=\left(1,1,1,\sqrt{\tfrac32}\right)^\top,
h'=\left(\sqrt{\tfrac23},1,\sqrt{\tfrac23},1,\sqrt{\tfrac23}\right)^\top,
 \text{ with }
\|X-w'h'^\top\|_1 \approx 6.899<7.
$$ 
Therefore, in contrast to~\eqref{eq:l1lrageneral}, we cannot assume, without loss of generality, that an optimal solution of rank-one L1-NMF with a binary input matrix is itself binary. 
This counterexample shows that the NP-hardness claim of L1-NMF as a consequence of the NP-hardness of the unconstrained case in~\cite[Theorem 6.13]{gillis2020nonnegative} was made incorrectly. 


Let us now provide a proof that rank-one L1-NMF is NP-hard, even when the input matrix is restricted to be binary, that is, $X\in\{0,1\}^{m\times n}$. 
Our proof proceeds by a polynomial-time reduction from the decision version of the problem~\eqref{eq:l1lrageneral}. 
We begin by formally stating the decision versions of both problems. 
\smallskip

\begin{mdframed}
\noindent
\emph{Problem} L1-LRMA

\noindent
Input: a natural number $D>0$ and a matrix $M\in\{\pm1\}^{s\times t}$. 

\noindent
Question: does there exist $u\in\mathbb{R}^{s}$ and $v\in\mathbb{R}^{t}$ such that $\|M-uv^\top \|_1\leq D$?
\end{mdframed}

\smallskip

\begin{mdframed}
\noindent
\emph{Problem} L1-NMF

\noindent
Input: a natural number $T>0$ and a matrix $X\in\{0,1\}^{m\times n}$. 

\noindent
Question: does there exist $w\in\mathbb{R}_+^{m}$ and $h\in\mathbb{R}_+^{n}$ such that $\|X-wh^\top \|_1\leq T$?
\end{mdframed}

\smallskip

As mentioned above, L1-LRMA is NP-hard by a polynomial-time reduction from \textsc{Max-Cut}~\cite{gillis2018complexity}.
Let us now reduce L1-LRMA to L1-NMF in polynomial-time. 
Given an instance $(M,D)$ of L1-LRMA, we construct an instance $(X,T)$ of L1-NMF as follows: 
\begin{itemize}
    \item $T=st+D$,
    \item $X\in\{0,1\}^{2s\times 2t}$ is defined blockwise: for each $(i,j)\in\{1,\dots,s\}\times\{1,\dots,t\}$,
    \[
    X(2i-1\!:\!2i,\;2j-1\!:\!2j)=
    \begin{cases}
    \begin{pmatrix}
        1 & 0\\
        0 & 1
    \end{pmatrix}, & \text{if } M_{ij}=1,\\[3mm]
    \begin{pmatrix}
        0 & 1\\
        1 & 0
    \end{pmatrix}, & \text{if } M_{ij}=-1.
    \end{cases}
    \]
\end{itemize}
Each entry of $M$ is encoded by a $2\times2$ binary block in $X$. 
The reduction mimics the two possible signs, $\{\pm1\}$, using the two 2-by-2 permutation matrices. 
This construction of $X$ allows us to ``restore'' the property that a binary optimal solution for rank-one L1-NMF exists, which is not true in general (see the example above). 

\begin{theorem}\label{lem:complexityyesyes}
The instance $(X,T)$ is a yes-instance of L1-NMF if and only if $(M,D)$ is a yes-instance of L1-LRMA. 
Hence rank-one L1-NMF is NP-hard. 
\end{theorem}

\begin{proof}

\noindent
\textbf{The \textit{if} part.}
Suppose first that $(M,D)$ is a yes-instance of L1-LRMA. Then there exist $u\in\mathbb{R}^s$ and $v\in\mathbb{R}^t$ such that $\|M-uv^\top \|_1\leq D$.
By~\cite{gillis2018complexity}, we may assume without loss of generality that $u\in\{\pm1\}^s$ and $v\in\{\pm1\}^t$. We define $w\in\mathbb{R}_+^{2s}$ and $h\in\mathbb{R}_+^{2t}$ blockwise as follows:
\[
w_{2i-1:2i}=
\begin{cases}
(1,0)^\top, & \text{if } u_i=1,\\
(0,1)^\top, & \text{if } u_i=-1,
\end{cases}
\qquad
h_{2j-1:2j}=
\begin{cases}
(1,0)^\top, & \text{if } v_j=1,\\
(0,1)^\top, & \text{if } v_j=-1.
\end{cases}
\]
Clearly, $w\geq 0$ and $h\geq 0$. For each $(i,j)$, let $E_{ij}$ denote the error of the $(i,j)$th block in $X$, that is,
\[
E_{ij}=\left\lVert X(2i-1\!:\!2i,\;2j-1\!:\!2j)-w_{2i-1:2i}h_{2j-1:2j}^\top\right\rVert_1. 
\]
By definition, the rank-one block $w_{2i-1:2i}h_{2j-1:2j}^\top$ contains exactly one entry equal to $1$ and three equal to $0$.
There are two possibilities:
\begin{itemize}
    \item if $M_{ij}=u_i v_j$, this entry equal to $1$ is located at an entry where the block
    $X(2i-1:2i,2j-1:2j)$ also has an entry equal to $1$, hence $E_{ij}=1$,
    \item if $M_{ij}\neq u_i v_j$, this entry equal to $1$ is located at an entry where the block
    $X(2i-1:2i,2j-1:2j)$ has an entry equal to $0$, hence $E_{ij}=3$. 
\end{itemize}
Since $M_{ij},u_iv_j\in\{\pm1\}$, we have
\[
|M_{ij}-u_iv_j|=
\begin{cases}
0, & \text{if } M_{ij}=u_iv_j,\\
2, & \text{if } M_{ij}\neq u_iv_j.
\end{cases}
\]
Therefore,
\[
E_{ij}=1+|M_{ij}-u_iv_j| \qquad \text{for all } (i,j).
\]
Summing over all blocks gives
\[
\|X-wh^\top \|_1 = \sum_{i=1}^s\sum_{j=1}^t E_{ij}
= st + \|M-uv^\top \|_1
\leq st + D = T.
\]
Hence, $(X,T)$ is a yes-instance of L1-NMF.

\medskip
\noindent
\textbf{The \textit{only if} part.} Suppose now that $(X,T)$ is a yes-instance of L1-NMF, that is,
there exist $w\in\mathbb{R}_+^{2s}$ and $h\in\mathbb{R}_+^{2t}$ such that $\|X-wh^\top \|_1\le T$.
With $u\in\mathbb{R}^{s}$ and $v\in\mathbb{R}^{t}$ defined by
\[
u_i = w_{2i-1}-w_{2i},\quad \text{for }i=1,\ldots,s, 
\quad \text{ and }\quad
v_j = h_{2j-1}-h_{2j},\quad \text{for }j=1,\ldots,t,
\]
we show that
\[
\lVert M-uv^\top\rVert_1\leq D.
\]

\begin{itemize}
    \item When $M_{ij}=1$, the error of the $(i,j)$th block satisfies
\begin{eqnarray*}
E_{ij} & = & |1-w_{2i-1}h_{2j-1}| + w_{2i-1}h_{2j} + w_{2i}h_{2j-1} + |1-w_{2i}h_{2j}| \\
& \geq & |1-w_{2i-1}h_{2j-1}+w_{2i-1}h_{2j}| + |1-w_{2i}h_{2j}+w_{2i}h_{2j-1}| \\
& = & |1-w_{2i-1}v_j| + |1+w_{2i}v_j|.
\end{eqnarray*}

If $v_j\geq 0$, then $w_{2i}v_j\geq 0$ and we obtain
\[
E_{ij} \geq |1-w_{2i-1}v_j| + 1+w_{2i}v_j \geq  1 + |1-w_{2i-1}v_j+w_{2i}v_j| = 1 + |1-u_iv_j|.
\]
If $v_j\leq 0$, then $-w_{2i-1}v_j\geq 0$ and we obtain
\[
E_{ij} \geq 1-w_{2i-1}v_j + |1+w_{2i}v_j| \geq  1 + |1+w_{2i}v_j-w_{2i-1}v_j| = 1 + |1-u_iv_j|.
\]
Hence, when $M_{ij}=1$, we have $E_{ij} \geq 1 + |M_{ij}-u_i v_j|$.
\smallskip
\item When $M_{ij}=-1$, the error of the $(i,j)$th block satisfies
\begin{eqnarray*}
E_{ij} & = & w_{2i-1}h_{2j-1} + |1-w_{2i-1}h_{2j}| + |1-w_{2i}h_{2j-1}| + w_{2i}h_{2j} \\
& \geq & |1-w_{2i-1}h_{2j}+w_{2i-1}h_{2j-1}| + |1-w_{2i}h_{2j-1}+w_{2i}h_{2j}| \\
& = & |1+w_{2i-1}v_j| + |1-w_{2i}v_j|.
\end{eqnarray*}

Using the same discussion on the sign of $v_j$, we have that
$$E_{ij} \geq 1+|1+u_iv_j| = 1 + |M_{ij}-u_i v_j|.$$
\end{itemize}
Summing over all the blocks leads to
\[
\|X-wh^\top \|_1
= \sum_{i=1}^s\sum_{j=1}^t E_{ij}
\geq  st + \|M-uv^\top \|_1.
\]
Since $\|X-wh^\top \|_1\leq T=st+D$, it follows that $\|M-uv^\top \|_1\leq D$.
Hence, $(M,D)$ is a yes-instance of L1-LRMA.
\end{proof}




\section{\revise{Sparse input and zero solution in scalar LAD subproblems}} 
\label{sec:sparsity_inducing}

\revise{When fixing all the entries in the factors but one in the L1-NMF formulation,} each one-variable subproblem consists of minimizing a piece-wise linear function of the form
\begin{equation}
    f(\alpha)=  \|x-\alpha y\|_1=\sum_{s=1}^m \lvert x_s -\alpha y_s \rvert=\sum_{s=1}^m \lvert y_s \rvert \left\lvert \frac{x_s}{y_s} -\alpha \right\rvert, 
     \label{eq:gen_scalar_prob}
\end{equation}
\revise{where $x$ represents a possibly negative residual vector and $y$ is a nonnegative vector. The derivation of the specific values of the coefficient is further investigated in Section~\ref{sec:l1_for_sparse}.} We show an example \revise{of piece-wise linear function} in Figure~\ref{fig:ex_piecewise}. The function in~\eqref{eq:gen_scalar_prob} is a convex, nondifferentiable function, where at least one global minimum coincides with one of the breakpoints $x_s/y_s$. 
\begin{figure}[H]
\centering
\begin{tikzpicture}[scale=0.8]

\draw[->] (-6.5,0) -- (8.5,0) node[below right] {$\alpha$};
\draw[->] (0,0) -- (0,5.5) node[left] {$f(\alpha)$};

\draw[thick, blue]
(-4,5) -- (-2,3) -- (0,1.8) -- (2,1) -- (4,1.6) -- (6,2.4) -- (8,4);

\foreach \x/\lab in {
-4/{$\frac{x_1}{y_1}$},
-2/{$\frac{x_2}{y_2}$},
 0/{$\frac{x_3}{y_3}$},
 2/{$\frac{x_4}{y_4}$},
 4/{$\frac{x_5}{y_5}$},
 6/{$\frac{x_6}{y_6}$},
 8/{$\frac{x_7}{y_7}$}
}{
\draw[fill=blue] (\x,0) circle (0.06);
\node[below] at (\x,0) {\lab};
}

\foreach \x/\y in {
-4/5,
-2/3,
 0/1.8,
 2/1,
 4/1.6,
 6/2.4,
 8/4
}{
\draw[fill=blue] (\x,\y) circle (0.06);
}

\end{tikzpicture}

\caption{Example of a one-dimensional nonnegative least absolute deviation (LAD) problem in~\eqref{eq:gen_scalar_prob}.}
\label{fig:ex_piecewise}
\end{figure}
The LAD problem~\eqref{eq:gen_scalar_prob} has an optimal zero solution if the slope of the function changes sign after a nonpositive breakpoint. 
Since $y\geq0$, a \revise{nonpositive} breakpoint $x_s / y_s$  corresponds to a \revise{nonpositive} numerator $x_s$, and we define $\mathcal{N}_{x}=\{s \ | \ x_s\leq0 \} \subseteq \{ 1,\dots,m \}$ as the set of indices corresponding to the negative and zero breakpoints. 
The corresponding LAD problem has a zero solution if
\begin{equation}
    \sum_{s \in \mathcal{N}_{x}} y_s \quad > \sum_{s \in \{1,\dots,m\} \backslash \mathcal{N}_{x}} y_s , 
    \label{eq:sparse_cond}
\end{equation}
\revise{which encodes the condition on the slope for an optimal breakpoint; for more details, see step~\ref{lineindex} in the weighted median algorithm, Algorithm~\ref{alg:constr_w_m}, described in Section~\ref{sec:l1_for_sparse}.}
If the data vector $x$ is sparse, the set $\mathcal{N}_{x}$ is larger than its complement $\{1,\dots,m\} \backslash \mathcal{N}_{x}$, and the condition in~\eqref{eq:sparse_cond} is more likely to be satisfied. 

\revise{The optimal solution of the scalar LAD is connected to the critical points of L1-NMF. In fact,} any critical point of L1-NMF has to be a coordinate-wise minimizer, that is, every entry of $W$ and $H$ has to be an optimal solution of the corresponding LAD subproblem of the form~\eqref{eq:gen_scalar_prob}.  
Let us then investigate the probability that the one-dimensional nonnegative LAD regression problem~\eqref{eq:gen_scalar_prob} has a zero solution, depending on the sparsity of the data. 
As we will see, \revise{under the assumption that the vectors at hand are Bernoulli distributed,} the sparser the data, the larger the probability that the LAD problem has a zero solution. \revise{Our aim is to provide an intuition on the behavior we observed numerically: sparser data yields significantly sparser L1-NMF factors with respect to FroNMF. Of course, our analysis relies on the assumption that the data and one factor are Bernoulli distributed, which is not realistic in practice since a)~the data usually follows more complex, continuous distributions, and  
b)~the factors of NMF align with the data, as they try to approximate it, hence are not independent of the data.}    

We proceed in our analysis by adding some probabilistic structure to $x$ and $y$, assuming they are binary vectors in $\{0,1\}^m$, sampled from a Bernoulli distribution with parameter $p$. Therefore, $x$ and $y$ have a probability $p$ of being positive and $(1-p)$ of being zero. Moreover, we are restricting $x$ to be nonnegative, thus ensuring that all the breakpoints are nonnegative. \revise{This situation corresponds to rank-1 L1-NMF, where $x$ represents a column of the matrix $X$}, however, when the factorization rank is larger than 1, $x$ might contain negative entries, resulting in an even larger number of nonpositive breakpoints. 
We aim to understand how the probability of having an optimal non-zero solution to the LAD problem and to the least squares problem 
changes with respect to the parameter $p$ that regulates the sparsity in the data. That is, we want to estimate the probability that the optimal solutions 
\begin{equation} \label{eq:l1_l2_comp}
    \alpha_2=\arg\min_{\alpha\geq 0} \|x-\alpha y\|_2, \quad \text{and} \quad \alpha_1=\arg\min_{\alpha\geq 0} \|x-\alpha y\|_1
\end{equation}
have a non-zero value as a function of $p$. Note that $\alpha_1$ and $\alpha_2$ are not necessarily unique; see the proof of Lemma~\ref{lem:explicit_prob}  below for the details. 
This allows us to have a better understanding of the probability of one entry of the L1-and FroNMF to be set to zero. 
The following lemma gives a simple characterization of the probabilities $\mathbb{P}(\alpha_2>0)$ and $\mathbb{P}(\alpha_1>0)$. 
\begin{lemma}
    Let $x, y\in\{0,1\}^m$ be generated from a Bernoulli distribution which takes value 1 with probability $p \in [0,1]$ and value zero with probability $1-p$ 
    as inputs for the problems~\eqref{eq:l1_l2_comp}.  Then  
\begin{equation}
\begin{aligned}
    \mathbb{P}(\alpha_1>0) &  = (1-p)^{m} + \sum_{k=1}^mp^k (1-p)^{m-k}\binom{m}{k}\left(\sum_{i=0}^{\lfloor \frac{k}{2} \rfloor}(1-p)^ip^{k-i}\binom{k}{i}\right).
    \\
    \mathbb{P}(\alpha_2>0) 
     & = (1-p)^{m} + \sum_{k=1}^mp^k (1-p)^{m-k}\binom{m}{k} \left( 1 - (1-p)^k \right).
\end{aligned}
\label{eq:explicit_prob}
\end{equation}
\label{lem:explicit_prob}
\end{lemma}
\begin{proof} 
If $y=0$, then any nonnegative $\alpha_i$ for $i=1,2$ is an optimal solution, so we consider  $\mathbb{P}(\alpha_i>0) = 1$. 
Hence
\[
\mathbb{P}(\alpha_i>0) = 
\mathbb{P}(y=0) 
+ \mathbb{P}( \alpha_i>0 \ | \ y \neq 0). 
\]
Now, let us consider the two cases separately. 
For the L1 norm ($i=1$), if $y$ is not the zero vector, looking at Algorithm~\ref{algo_cons_weightedmedian}, one can check that $\alpha_1>0$ if $\sum_{j \in \supp(y)} x_j \geq \frac{|\supp(y)|}{2}$, that is, $x$ has at least as many 1's as 0's on the support of $y$, defined as the set of indices corresponding to non-zero entries in $y$ and denoted $\supp(y)$. 
Note that when $k =|\supp(y)|$ is even and $x(\supp(y))$ contains exactly $\frac{k}{2}$ zeros, any value $\alpha_1 \in [0,1]$ is optimal, hence we consider it positive with probability 1.
Finally  
\begin{equation}
\begin{aligned}
    \mathbb{P}(\alpha_1>0)=&\mathbb{P}(y=0)+ \mathbb{P}\left(\sum_{j \in \supp(y)} x_j \geq  \frac{|\supp(y)|}{2} \right)\\
    =&\mathbb{P}(y=0) + \sum_{k=1}^m \mathbb{P}\left( \sum_{j \in \supp(y)} x_j \geq \frac{k}{2} \ \middle|  \ |\supp(y)|=k\right) \mathbb{P}\left(|\supp(y)|=k\right).
\end{aligned} 
\label{eq:step1_lem_prob}
\end{equation}
Since $|\supp(y)|$ and $\sum_{j \in \mathcal K} x_j$ for any subset $\mathcal K$ of $\{1,2,\dots,m\}$  
follow a binomial distribution, we obtain~\eqref{eq:explicit_prob} for the L1 norm. 


For the L2 norm ($i=2$), when $y \neq 0$, the optimal solution has a unique closed form expression, $\alpha_2 = \frac{x^\top y}{\|y\|_2^2}>0$. 
Hence $\alpha_2 > 0$  as long as $x(\supp(y)) \neq 0$. We get  
\begin{equation}
\begin{aligned}
    \mathbb{P}(\alpha_2>0)=&\mathbb{P}(|\supp(y)|=0)+ \mathbb{P}\left(\sum_{j \in \supp(y)} x_j \geq 1 \right). 
\end{aligned} 
\label{eq:l2_prob}
\end{equation}
Following the same steps as above, and using the binomial distribution,  we obtain~\eqref{eq:explicit_prob} for the L2 norm.
\end{proof}

It is interesting to analyze the behavior of $\mathcal{P}(\alpha_i >0)$ for $i=1,2$ as the dimension $m$ increases. 
Figure~\ref{fig:prob_alpha} shows $\mathbb{P}(\alpha_1 >0) $ and $\mathbb{P}(\alpha_2 >0) $  as functions of the dimension $m$ for different choices of the  Bernoulli parameter~$p$. 
For $i=1$, we observe that, as the dimension $m$ increases,  
 the probability of the LAD solution being positive (i) goes to zero for  $0\leq p<\frac{1}{2}$, (ii) goes to 1/2 for $p=1/2$, and (iii) goes to one for $p > 1/2$. 
For the least squares problem ($i=2$), 
$\mathbb{P}(\alpha_2 >0)$ goes to 1 as $m$ increases for any $p > 0$.  
In fact, one can prove that, for $p > 0$, since both $|\supp(y)|$ and $\sum_{j\in \supp(y)} x_j$ for a fixed support $\supp(y)$ are binomials, 
\begin{equation}
    \lim_{m \to \infty} \mathbb{P}(\alpha_2>0) = 1 \quad \text{and} \quad \lim_{m \to \infty} \mathbb{P}(\alpha_1>0) = \begin{cases} 1 & \text{ if  } p > \frac{1}{2}\\
\frac{1}{2} & \text{ if  } p = \frac{1}{2}\\
0 & \text{ if  } p < \frac{1}{2}
\end{cases} . 
\label{eq:bin_lem}
\end{equation}

Therefore, for large datasets, FroNMF is likely to generate dense solutions, even for sparse data. Of course, the basis vectors are not generated randomly and will align with the data. However, our result shows that for the randomly generated data matrix $X$ and basis vector $W$ \revise{from the Bernoulli distribution}, the corresponding FroNMF optimal factor $H$ will be dense with high probability, as long as $m$ is sufficiently large. 
This is somewhat odd: randomly generating basis vectors $W$  will provide dense activations $H$. 
On the contrary, L1-NMF will generate sparse solutions, even when the dimension increases. This will be confirmed in the numerical experiments in Section~\ref{sec:num_res}.

\begin{figure}
\begin{minipage}[h]{.49\linewidth}
    \centering
       \resizebox{\linewidth}{!}{\input{Tikz/alpha2_p0}}
\end{minipage}
\begin{minipage}[h]{.49\linewidth}
    \centering
    \resizebox{\linewidth}{!}{ \input{Tikz/alpha1_p0}}
\end{minipage}
 \caption{Probability of the least squares solution (on the left) and LAD solution (on the right) in~\eqref{eq:l1_l2_comp} to be greater than zero when each component of the inputs is randomly sampled from a Bernoulli distribution in $\{0,1\}$ with probability $p$ to be 1. } 
    \label{fig:prob_alpha}
\end{figure}

\section{New L1-NMF model: wL1-NMF}
\label{sec:w_ell_1_nmf}
Despite the sparsity-inducing property of L1-NMF being meaningful in several contexts, it might lead to biased factorizations when the zeros in the data contain \textit{false zeros}~\cite{loeys2012analysis, blasco2019does, lopez2023zero}. False zeros appear, for example,  in imaging mass spectrometry where the amount of data is so large that the instruments used for the measurements report only values above some threshold (\revise{known as clipping}), in order to achieve sparse representations~\cite{moens2025preserving,gonzalez2023nectar}. In such applications, a good model that aims to recover the missing information from the observed data should approximate false zeros by non-zero values. At the same time, zeroing out the probability of having a zero value in the low-rank approximation does not take into account that zeros represent small values, thereby losing an important piece of information. 
In word-by-document count data, a zero entry does not necessarily mean that the word is not associated with the topics discussed by the corresponding document. This is especially critical when the documents are short (e.g., tweets). 
Another example is in recommender systems in which   missing values are not random and correspond to a mixture of negative and unknown preference~\cite{wang2018modeling}. Since preferences are typically nonnegative, 
it makes sense to assign zeros to missing entries, meaning a negative opinion on the corresponding items. These entries carry some useful information, namely a potential negative preference: someone did not consume an item because it was not likely that he/she would enjoy it. 


Assuming that zeros cannot be trusted as much as the positive entries but still carry some meaningful information, wL1-NMF, introduced in~\eqref{eq:weight_l1}, 
downweights the importance of the zero entries  using a penalty parameter $\lambda$. Let us recall the formulation: 
\begin{equation}
    \ell(X,W,H,\lambda) := \sum_{(i,j) \in \kappa^+} \lvert X -WH \rvert_{i,j} + \lambda \sum_{(i,j) \in \kappa^0} \lvert WH \rvert_{i,j},
    \label{eq:weight_l1_2}
\end{equation}
where $\lambda \in [0,1]$, $\kappa^+=\left\{ (i,j) \ | \ X_{i,j}>0 \right\}$,  
    and 
    $\kappa^0=\left\{ (i,j) \ | \ X_{i,j}=0 \right\}$.
This corresponds to a low-rank estimator that follows two Laplace distributions with different variances: the first for the positive values and the second for the zeros; see below for more details.  

\paragraph{Statistical interpretation of wL1-NMF}

The minimizer of the L1-NMF model in~\eqref{eq:l1NMF} is the maximum likelihood estimator when the $(i,j)$th entry of $X$ is the realization of a random variable $\tilde{X}_{i,j}$ that depends on the parameter $(\Hat{W}\Hat{H})_{i,j}$ and is corrupted by additive Laplacian noise, that is, 
\begin{equation}
    \tilde{X}_{i,j} =(\Hat{W}\Hat{H})_{i,j}+\tilde{N}_{i,j}, \quad \text{where } \quad \Hat{W}\geq0, \ \Hat{H}\geq0, \ \text{and} \ \tilde{N}_{i,j} \sim \mathcal{L}(0,\sigma),
    \label{eq:prob_struct}
\end{equation}
where $\tilde{N}_{i,j}$ follows a zero mean Laplace distribution with the diversity $\sigma$. \revise{By minimizing the negative logarithm of the likelihood associated to the model in~\eqref{eq:prob_struct} with respect to $W$ and $H$, we obtain estimates $W^*$ and $H^*$ for the unknown parameters $\Hat{W}$ and $\Hat{H}$ by solving the L1-NMF model. }  

\revise{Considering a noise that follows a continuous distribution makes the model not suitable for} realization of $\tilde{X}$ which contains many zero entries. When $X_{ij}$ is zero, we can deduce from~\eqref{eq:prob_struct} that the corresponding entries in the estimate $(W^* H^*)_{ij}$ must be close to the distribution of the noise $\Tilde{N}_{ij}$. Since the Laplace distribution has a sharp peak at zero, this results in a penalization of the magnitude of $(W^* H^*)_{ij}$. We propose to downweight the penalization effect by rescaling the noise affecting the components in $\kappa^0$, corresponding to zeros in $\Tilde{X}$. The resulting likelihood is 
\begin{equation}
     \ell_1(\Tilde{X};\Hat{W}\Hat{H},\sigma)=\prod_{(i,j)\in \kappa^+ } \frac{1}{2 \sigma} e^{-\frac{1}{\sigma}\lvert \Tilde{X}_{i,j} - (\Hat{W}\Hat{H})_{i,j} \rvert } \prod_{(i,j)\in \kappa^0} \frac{\lambda}{2 \sigma} e^{-\frac{\lambda}{\sigma} (\Hat{W}\Hat{H})_{i,j} }, 
    \label{eq:posteriori}
\end{equation}
where $\lambda$ is a parameter smaller than 1. Taking the negative logarithm and removing the constant terms, we obtain the wL1-NMF objective function~\eqref{eq:weight_l1_2}. Note that a larger diversity means that the Laplace distribution has a less sharp peak at zero. Equivalently, larger diversities, forced by smaller values of $\lambda$, reduce the probability of $(\Hat{W} \Hat{H})_{i,j}$ for $(i,j) \in \kappa^0$ to be zero. If the data at hand are heterogeneous, one might consider using different $\lambda$ parameters for each column (or row) of $\Tilde{X}$. Moreover, if we cannot trust the zeros or we want to obtain denser solutions, smaller values of $\lambda$ should be used since they correspond to increasing the variability in the prior Laplace distribution of the entries in $\kappa^0$. \revise{On the other hand, when most of the zeros are true, higher values of the parameter $\lambda$ should be preferred.}

\section{Efficient CD algorithm for L1-NMF with sparse data} 
\label{sec:l1_for_sparse}


CD for FroNMF, also known as HALS, is among the most efficient methods for solving NMF problems~\cite{cichocki2007hierarchical, gillis2012accelerated}. In this section, we introduce \name, a novel CD scheme for wL1-NMF~\eqref{eq:weight_l1_2}, that scales linearly with the number of non-zero entries of the original matrix. 
Every iteration of our method is equivalent to an iteration of the standard CD algorithm, but the computational cost is substantially reduced by exploiting the sparsity and nonnegativity in the data. 
As with CD, \name\ sequentially optimizes over the entries of $W$ and $H$. 
Let us focus on the subproblems for the entries of $H$, the subproblems for $W$ are the same, by symmetry.  
Let 
\begin{equation}
    \kappa^+_j=\left\{ i \ | \ (i,j) \in \kappa^+ \right\} \quad \text{and } \quad \kappa^0_j=\left\{ i \ | \ (i,j) \in \kappa^0 \right\},
    \label{eq:def_set_kappa_q_j}
\end{equation}
be the indices of the positive and zero entries for each column of the matrix $X$, where $\kappa^+$ and $\kappa^0$ are defined after~\eqref{eq:weight_l1_2}. The subproblem for the $j$th column $H_{:,j}$ of $H$ becomes 
\begin{equation}
    \argmin_{H_{:,j}\geq 0} \sum_{s \in \kappa^+_j} \lvert X -WH \rvert_{s,j} + \lambda \sum_{s \in \kappa^0_j} \lvert WH \rvert_{s,j}, \quad \text{for } j=1,\dots,n.
    \label{eq:sub_H_weigh}
\end{equation}
Each subproblem in one column of $H$ is independent; thus, they can be solved in parallel. We further split the problem into $r$ one-dimensional LAD problems, one for each entry $H_{i,j}$ of $H_{:,j}$, for $i=1,\dots,r$, and we obtain 
\begin{equation}
   \argmin_{H_{i,j}\geq 0} \sum_{s \in \kappa^+_j} \lvert X_{s,j}-\sum_{t=1, \ t\neq i}^r (W_{s,t}H_{t,j})-W_{s,i}H_{i,j} \rvert + \lambda H_{i,j}  \sum_{s \in \kappa^0_j} W_{s,i},
\label{eq:sparse_l1_H_explicit}
\end{equation}
that can be simplified to 
\begin{equation}
    \argmin_{H_{i,j}\geq 0} \sum_{s \in \kappa^+_j}\lvert x_s - y_s H_{i,j} \rvert +  \lvert 0 - c H_{i,j} \rvert,
    \label{eq:fin_subprob}
\end{equation}
where 
\begin{equation}
    x_s= X_{s,j}-\sum_{t=1, \ t\neq i}^r (W_{s,t}H_{t,j}), \quad y_s=W_{s,i}, \quad c=\lambda \sum_{s \in \kappa^0_j} W_{s,i}.  
    \label{eq:a_b_c_def}
\end{equation}

\revise{Equation~\eqref{eq:fin_subprob} represents a piecewise linear function of the same form as the one in~\eqref{eq:gen_scalar_prob}; thus.} A naive strategy for finding a global minimum is to evaluate it at all the breakpoints and take the one that corresponds to the lowest value, or zero if the optimal breakpoint is negative. \revise{Since the number of breakpoints is equal to the terms in the summation in~\eqref{eq:fin_subprob},} this approach has a complexity of \revise{$\mathcal{O}(|\kappa^+_j|^2)$} and is not used in practice. A faster approach, running in \revise{$\mathcal{O}(|\kappa^+_j|\log(|\kappa^+_j|))$} exists, consists of identifying an optimal solution at the breakpoint after which the sign of the slope of the function changes. 
The slopes of the function at every breakpoint are obtained from the vector $y$. The algorithm implementing this strategy is known as the constrained weighted median algorithm~\cite{edgeworth1888xxii}; see Algorithm~\ref{alg:constr_w_m}. 
We also mention that there exists a more involved algorithm for computing a solution to the weighted median problem running in linear time \revise{$\mathcal{O}(|\kappa^+_j|)$}~\cite{gurwitz1990weighted}. \revise{It is worth mentioning that when two breakpoints yield the same function value, the interval defined by the two is a set of optimal solutions of the piece-wise linear problem. Therefore, we make the choice of choosing the smaller breakpoint but one could choose any point in the interval of global minimizer. Nevertheless, tie situations are extremely rare in practical application where the input data are real numbers. It may however impact the performance of the algorithm if the data exhibit some regular structure such as binary or integer entries, this is a topic for further research.} 

\begin{algorithm}
	\caption{\quad \texttt{$\alpha$ = constrained\_weighted\_median(x,y)}}\label{algo_cons_weightedmedian}
	\begin{algorithmic}[1]
		\item INPUT: $x\in \mathbb{R}^{m}$, $y\in \mathbb{R}^{m}$
		\item OUTPUT: $\arg\min_{\alpha\geq 0}\sum_s |x_s-\alpha y_s|$
		\State $S\gets \{\frac{x_s}{y_s}  \ | \ y_s \neq 0\}$
		\State $[S,\text{indices}]\gets \texttt{sort}(S)$ \label{linesort}
		\State $y \gets y(\text{indices})$
		\State $k \gets$ smallest index such that $\sum_{s=1}^k y_s\geq \frac{\sum_{s=1}^m y_s}{2}$ \label{lineindex}
		\State $\alpha\gets\max(0,S_k)$
	\end{algorithmic}
    \label{alg:constr_w_m}
\end{algorithm}

\revise{The novelty of our approach is that,}
exploiting the nonnegativity of the factors $W$ and $H$ in Equation~\eqref{eq:sparse_l1_H_explicit}, we group the contributions of all the components in $\kappa^0_j$, \revise{allowing us to rewrite the one-dimensional subproblem in $H_{i,j}$ as the sum of $|\kappa^+_j|+1$ terms. This means that the number of breakpoints is reduced from $m$ to $|\kappa^+_j|+1$}. Therefore, the constrained weighted median algorithm can be employed to obtain one global minimum in $\mathcal{O}(|\kappa^+_j|\log(|\kappa^+_j|))$ rather than $\mathcal{O}(m\log(m))$ as in the standard CD approach. If $|\kappa^+_j| \ll m$, then this modification consistently reduces the computational cost of each one-variable subproblem while producing the same sequence of iterates as the original CD scheme. 
It is interesting and important to note that it would not be possible to take advantage of the sparsity in the unconstrained case, as the signs of the entries of $WH$ are unknown; hence, one needs to compute all entries of $WH$, leading to a computational cost of $\mathcal O(mnr)$; see also Section~\ref{sec:compcost}.  
In other words, to the best of our knowledge, L1-LRMA cannot be solved with CD with complexity linear in the number of non-zero entries in $X$.

We have seen in Section~\ref{sec:sparsity_inducing} how the sparsity of the data affects the sparsity of the L1-NMF solution. 
Specifically, if the condition in~\eqref{eq:sparse_cond} is satisfied, the nonnegative LAD problem in~\eqref{eq:gen_scalar_prob} has an optimal zero solution. Using the same notation as in~\eqref{eq:a_b_c_def}, and defining $\mathcal{N}_{x}^-=\{s=1,\dots,m \ | \ x_s < 0 \}$, we rewrite the condition in~\eqref{eq:sparse_cond} to have an optimal solution $H_{i,j}=0$ for the subproblem in~\eqref{eq:fin_subprob}, that is, 
\begin{equation*}
    \sum_{s \in {\mathcal{N}}_{x}^-} W_{s,i} + \lambda \sum_{s \in \kappa_{j}^0} W_{s,i} > \frac{1}{2} \sum_{s \in \kappa_{j}^+} W_{s,i} +  \frac{\lambda}{2} \sum_{s \in \kappa_{j}^0} W_{s,i}.
    \label{eq:pre_sparse_cond_weight_pre}
\end{equation*}
Equivalently, we have
\begin{equation}
    \sum_{s \in {\mathcal{N}}_{x}^-} W_{s,i} + \frac{\lambda}{2} \sum_{s \in \kappa_{j}^0} W_{s,i} > \frac{1}{2} \sum_{s \in \kappa_{j}^+} W_{s,i}.
    \label{eq:pre_sparse_cond_weight}
\end{equation}
The smaller the parameter $\lambda$, 
the smaller the number of indices satisfying the condition in~\eqref{eq:pre_sparse_cond_weight}. Consequently, \name\ with $\lambda<1$ produces denser factorizations as $\lambda$ becomes smaller. This will be illustrated in Section~\ref{sec:num_res}.

\subsection{Implementation details} Let us provide some implementation details for \name. 
First, the scalars $x_s$ for  $s=1,\dots,|\kappa^+_j|$ should not be computed using the explicit formula in~\eqref{eq:a_b_c_def}. In fact, let $(W^k,H^k)$ be the $k$th iterate. Defining $v_{s,j} = W^k_{s,:}H^k_{:,j} $ for every $s$ in $\kappa^+_j$ and for $j=1,\dots,n$, 
$x_s$ can be obtained with minimal cost from 
\begin{equation*}
    x_s = X_{s,j} - v_j + W^k_{s,i}H^k_{i,j}, \quad \text{for every } i=1,\dots,r.  
\end{equation*}
Similarly, after computing the new $H_{i,j}^{k+1}$,  $v_j$ is cheaply updated by
\begin{equation*}
    v_{s,j}\gets v_{s,j} + W^k_{s,i}(H_{i,j}^{k+1}-H^k_{i,j}), \quad \text{for every } s \text{ in } \kappa^+_j. 
\end{equation*}
We now give the full scheme for the update of $H$ in Algorithm~\ref{alg:H_update_full}. We use the simplified notation $\kappa$ to denote $\kappa^+_j$, where we omit the lower script $j$ and the super script $+$ for simplicity. 
\begin{algorithm}
	\caption{\quad \texttt{H = updateH($X$,$W^k$,$H^k$,$\lambda$)}}
	\begin{algorithmic}[1]
		\item INPUT: $X\in \mathbb{R}^{m\times n}$, $W^k\in \mathbb{R}^{m\times r}_+$, $H^k\in \mathbb{R}^{r\times n}_+$, $\lambda$
		\item OUTPUT: $H^{k+1}\in \mathbb{R}^{r\times n}_+$
        
		\State $s_i = \sum_{s=1}^{m} W^k_{s,i}$ \quad $\forall i = 1,\dots,r$ \label{line4}
        
         \State err = 0 
         
		\For{$j=1:n$}
		\State $\kappa = \{s \ |\ X_{s,j}>0\}$
		\State $v_j = W^k_{\kappa,:}H^k_{:,j}$ \label{line8}
		\For{$i=1:r$}
		\State $b \gets W^k_{\kappa,i}$ \label{line10}



        
		\State $a \gets X_{\kappa,j} - v_j + W^k_{\kappa,i}H^k_{i,j}$

		\State $c \gets \lambda \left(s_i - \sum_{s\in \kappa} W^k_{s,i} \right)$
		\State $H_{i,j}^{k+1} \gets \texttt{constrained\_weighted\_median}([a\ 0],[b\ c])$
        
		\State $v_j\gets v_j + W^k_{\kappa,i}(H_{i,j}^{k+1}-H^k_{i,j})$ \label{line14}

		\EndFor
                \State err $\gets$ err + $\| X_{\kappa,j} - v_j\|_1 - \revise{\lambda} e^\top v_j$ 

		\EndFor
        \State err $\gets$ err + $\revise{\lambda}(e^\top W)(He)$
	\end{algorithmic}
    \label{alg:H_update_full}
\end{algorithm}

The same algorithm applies to the update of $W^k$ observing that \revise{$\lVert X -WH \rVert=\lVert X^\top  -H^\top W^\top   \rVert$}, thus Algorithm~\ref{alg:H_update_full} with inputs $X^\top $, $(H^k)^\top $ is employed to get the new updated $(W^{k+1})^\top $. The overall procedure to compute the L1-NMF is given in Algorithm~\ref{alg:sparse_l1_nmf}. 
\begin{algorithm}
	\caption{\quad \texttt{[W,H] = \name(X,W$^0$,H$^0$,$\lambda$,maxiter)}}
	\begin{algorithmic}[1]
		\item INPUT: 
        nonnegative matrix $X\in \mathbb{R}^{m\times n}_+$, 
        the  $W^0 \in \mathbb{R}^{m\times r}$, $H^0 \in \mathbb{R}^{r \times n}$, 
        regularization parameter $\lambda$, 
        maxiter. 
        
		\item OUTPUT: $W\in \mathbb{R}^{m\times r}_+$ and $H\in \mathbb{R}^{r\times n}_+$
		\For{$k=1, \dots,$ maxiter}
        \vspace{0.1cm}
		\State $(W^{k+1})^\top =\texttt{updateH}(X^\top ,\revise{(H^k)^\top ,(W^k)^\top},\lambda )$
        \vspace{0.1cm}
		\State $H^{k+1}=\texttt{updateH}(X,W^{k+1},H^k,\lambda)$
		\EndFor
	\end{algorithmic}
    \label{alg:sparse_l1_nmf}
\end{algorithm}

\subsection{Computational cost of \name} 
\label{sec:compcost}


The most expensive operation of \name\ for updating one entry of $H$ in Algorithm~\ref{alg:H_update_full} is the constrained weighted median algorithm, which has complexity $\mathcal{O}(|\kappa^+_j|\log(|\kappa^+_j|))$, where $|\kappa^+_j|$ is the number of nonzero entries in the $j$th column of $X$. Thus, the update of all entries of $H$ has a total computational cost of $\mathcal{O}(r \sum_{j=1,\dots,n}|\kappa^+_j|\log(|\kappa^+_j|))$, which can be upper bounded by $\mathcal{O}(r \nnz(X) \log(\nnz(X)))$. 
The same reasoning applies for updating $W$. Hence, one iteration of \name\ costs $\mathcal{O}(r \nnz(X) \log(\nnz(X)))$ operations. 
If the original matrix $X$ is sparse, that is,  $\nnz(X) \ll mn$, then \name\ consistently reduces the computational cost of the  naive CD algorithm, which has a complexity of $\mathcal{O}(r mn \log(mn))$ operations. This difference will be illustrated in Section~\ref{sec:num_res}. 
To the best of our knowledge, \name\ is the first L1-NMF algorithm that scales linearly with $\nnz(X) \log(\nnz(X))$. 

\paragraph{Computing the L1-NMF objective} 
It is interesting to note that the L1-NMF objective (and similarly that of wL1-NMF) can be computed in $\mathcal O(\nnz(X)r)$ operations, explaining why it is possible to design algorithms running in a number of operations proportional to $\nnz(X)$. 
In fact, 
\begin{align*} 
\| X - WH \|_1 
= 
\sum_{(i,j) \in \kappa^0} \lvert WH \rvert_{i,j} + 
\sum_{(i,j) \in \kappa^+} \lvert X -WH \rvert_{i,j} 
= 
\|WH\|_1 + \sum_{(i,j) \in \kappa^+} \big( \lvert X -WH \rvert_{i,j} - (WH)_{i,j}   \big), 
\end{align*}
where $\|WH\|_1 = (e^\top W) (H e)$ can be computed in $\mathcal O(r(m+n))$ operations, while the \revise{remaining sum} needs to evaluate $\nnz(X)$ terms with cost $\mathcal O (r)$ for each term. 
Note that this would not be possible if $W$ and $H$ were unconstrained. 

\subsection{Convergence of \name} 

\name\ decreases the objective function at every iteration, and therefore the sequence of the function values will converge since it is bounded below. 
However, to guarantee the convergence of the sequence of the CD iterates to a coordinatewise stationary point of the nondifferentiable problem in~\eqref{eq:l1NMF}, each of the subproblems in $W$ and $H$ must have a unique global minimizer~\cite{tseng2001convergence}. This is not the case for wL1-NMF since the scalar subproblem for the update of $H_{i,j}^k$ in~\eqref{eq:gen_scalar_prob} might have multiple global minima. Specifically, when two or more breakpoints achieve the lowest function value, any point in the interval between the breakpoints is a global minimum. Therefore, \name\ might converge to a point that is not a coordinate-wise stationary point of wL1-NMF. Moreover, the absence of global convergence guarantees suggests that the initialization of the algorithm plays a fundamental role, as our experience confirms. 

\subsection{\name\ for binary factorizations}  

When $X$ is binary, one may wish to constrain $W$ and $H$ to be binary as well~\cite{zhang2007binary}, this is referred to as binary matrix factorization (BMF).  
Interestingly, if \name\ is initialized with binary factors, $W^0 \in \{0,1\}^{m \times r}$ and $H^0 \in \{0,1\}^{r \times n}$, and the input is binary, $X \in \{0,1\}^{m \times n}$, 
one can show that the iterates of \name\ will remain binary without enforcing it explicitly. The reason is that the break points in the LAD subproblems are 0, 1 or negative. Hence \name\ can be used as a BMF algorithm using the component-wise $\ell_1$ norm as the error measure. Comparing \name\ to existing BMF algorithms is out of the scope of this paper, and left for future works. 

\section{Numerical experiments}
\label{sec:num_res}

Our goal in this section is fourfold: 
\begin{enumerate}

\item[(i)] Show how \name\ allows a significant acceleration of CD for sparse data. 

    \item [(ii)] Show that the wL1-NMF is preferable to other NMF models (namely, FroNMF, KL-NMF, and L21-NMF) in situations where the data is affected by heavy tailed noise or outliers. 

    \item[(iii)] Compare \name\ with the state of the art algorithms for L1-NMF, in particular, 
    the  standard CD~\cite{ke2003robust,ke2005robust,gillis2011dimensionality}, 
    the BCD approach with Nesterov Smoothing (NS) by Guan et al.~\cite{guan2012mahnmf}, 
    and the projected subgradient (SUB) approach by  Rahimi et al.~\cite{rahimi2024projected}, showing the significant  computational advantage brought by \name\ when dealing with sparse data. 
    
\item[(iv)] Show that suitable choices of the penalization parameter $\lambda$ in wL1-NMF lead to significantly better results than L1-NMF ($\lambda = 1$).   
\end{enumerate}

We present numerical experiments on both synthetic and real datasets. In particular, we test our model and algorithm in noisy image compression, topic modeling, and matrix completion in the presence of false zeros.  

\paragraph{Setup} We use HALS to compute FroNMF, MU for KL-NMF following the implementation in~\cite{gillis2020nonnegative}, and MU for L21-NMF from~\cite{kong2011robust}. 
The smoothing parameter in the BCD approach from~\cite{guan2012mahnmf} with Nesterov smoothing is initialized at 0.5, and the number of inner iterations to compute each update of the factor is set to 30; the step size in the subgradient method is selected according to a linearly decaying rule as considered in~\cite{rahimi2024projected}, starting from an initial value of 0.01.

As explained in Section~\ref{sec:sparsity_inducing}, when dealing with sparse data, some initializations of the factors may yield overly sparse L1-NMF. According to our experience, a few iterations of HALS provide a reasonable starting point for computing L1-NMF. Therefore, if not otherwise specified, we initialize the algorithms solving L1-NMF or wL1-NMF  with 10 iterations of HALS. 
 HALS is itself initialized randomly. 
 All the results are averaged over 10 random initializations. Let $X$ be the data matrix and $(W^k,H^k)$ the approximation at step $k$ of any algorithm. If not otherwise specified, we stop each algorithm when the difference in the relative error, $\lVert X - W^k H^k \rVert_1 / \lVert X \rVert_1$, after two consecutive iterations, is below a tolerance of $10^{-6}$, or when a fixed time limit is reached. 
 The time limit is set according to the application being considered. Since the projected subgradient method is non-monotonic in decreasing the objective function, the only stopping criterion considered is the time limit. 

The algorithms are implemented in Julia version 1.11.4 (2025-03-10) on a 64-bit Samsung/Galaxy with 11th Gen Intel(R) Core(TM) i5 1135G7 @ 2.40GHz and 8 GB of RAM, under Windows 11 version 23H2. The code is available online from \url{https://github.com/giovanniseraghiti/wL1-NMF}. 

\revise{
\paragraph{Choice of the regularization parameter in the wL1-NMF model}
The choice of the regularization parameter $\lambda$  is crucial in practice, since it impacts two main aspects: the trust we place in the zeros in the data, and the sparsity we want to achieve in the wL1-NMF solution. A good regularization parameter should balance the two.

If we assume that the data contain false zeros, standard cross-validation (performed by holding out some known entries and treating them as missing, and picking the parameter that best estimates these hold-out values) might produce biased estimates. Consider a typical application in which zeros represent small, undetected values: if we withhold a portion of known, large entries, we bias the choice of 
$\lambda$ toward values that favor large predictions, while the actual goal might be the opposite. Moreover, in applications such as feature extraction and topic modeling, the quality of the solution is measured by the interpretability of the results, which is often difficult to quantify and may not coincide with lower reconstruction errors. Our experience shows that it is difficult to identify a single approach producing a satisfactory regularization parameter across different applications. For this reason, we select 
$\lambda$ by trial-and-error. This choice is also motivated by the fact that wL1-NMF solutions appear to be robust to small changes in $\lambda$; 
see for example Figure~\ref{fig:lambda_mc}. Therefore, small deviations from a good value do not appear to substantially affect the quality of the resulting factorization, making trial-and-error search a reasonable and practical strategy.
} 

\subsection{Synthetic data} 
We first show the computational advantage of using \name\ with respect to the original version of CD for L1-NMF on synthetic data. 
We generate different sparse $m\times n$ matrices with dimensions $(m,n)=(100,200),$ $(300, 400),$ $(500, 600), $ $(800, 1000)$ whose entries follow a uniform distribution in $[0,1]$. 
Then, we sparsify each matrix by randomly setting some of its entries to zero. 
We fix the rank to $r=20$ for all cases and run both algorithms for 30 iterations, ignoring any other stopping criteria. We specify that, in this first experiment, we are not interested in the quality of the approximations.  We focus on comparing the computational efficiency of the two versions of CD, which perform the same updates but in a different way. 
Table~\ref{tab:synt_comp} reports the average CPU time of \name\ and standard CD. 
It also reports the theoretical factor of gain of using \name\ over standard CD, which is the ratio between the asymptotic complexities of the two algorithms, that is, \begin{equation}
     \sigma=\frac{mn\log(mn)}{\nnz(X)\log(\nnz(X))}.
     \label{eq:comp_fact}
\end{equation}
The quantity $\sigma$ is the theoretical number of iterations that \name\ can perform in the time required for one standard CD iteration. 
We compare the theoretical factor of gain, $\sigma$, with the practical gain, which is the ratio between the average CPU times of CD and \name \ (last column of Table~\ref{tab:synt_comp}). 
\begin{table}[h!]
\centering
\begin{tabular}{ c || c | c | c | c | c  } 
 \toprule 
$m \times n$ & sparsity  & \name\  &  CD & $\sigma$ &  gain \\
  \midrule
     $100 \times 200$  \multirow{5}{*}{}& $25 \%$  & 0.31  & 0.40 & 1.37 & 1.29 \\
 &   $50 \%$  & 0.20 & 0.40 & 2.15 & 2.0 \\
  &  $80 \%$   & 0.08 &  0.39 & 5.97 & 4.88\\
  
   \hline
 $300 \times 400$  \multirow{5}{*}{}& $25 \%$  & 2.20   & 2.94 &1.37 & 1.33 \\
 &   $50 \%$  & 1.59 & 3.05 & 2.13 &  1.92\\
  &  $80 \%$   & 0.61 &  2.67 &5.80 & 4.38\\
  
   \hline
 $500 \times 600$  \multirow{5}{*}{}& $25 \%$  &  6.40  &  8.56 & 1.36 & 1.34 \\
 &   $50 \%$  & 4.34 &  8.85 & 2.12 & 2.04 \\
  &  $80 \%$   &  1.70 &  8.16 & 5.73 & 4.8 \\

   \hline
 $800 \times 1000$  \multirow{5}{*}{}& $25 \%$  &  20.53  &  27.68 & 1.36 & 1.35 \\
 &   $50 \%$  & 14.08 &  30.02 & 2.11 & 2.13 \\
  &  $80 \%$   & 4.75 &  27.65 & 5.67 & 5.82\\
 
 \bottomrule 
\end{tabular}
\caption{Comparison between \name\ and the original CD for L1-NMF on randomly generated data for different dimensions ($m \times n$) and sparsity of the input matrix. 
The third and fourth columns report the average CPU time per iteration for \name\ and CD, respectively, in seconds.  
The fifth and sixth column report the theoretical gain, $\sigma$ in~\eqref{eq:comp_fact}, and the practical  gain, resp. }   
\label{tab:synt_comp}
\end{table}

Table~\ref{tab:synt_comp} illustrates that \name\ follows closely the theoretically expected gain compared to the standard CD. 
In fact, when the sparsity increases from $25\%$ to $80\%$, the computational time decreases  approximately by a factor of four. 
On the contrary, the standard CD algorithm has a complexity that does not depend on the sparsity of the data, resulting in similar CPU times per iteration for each level of sparsity.

\subsection{MNIST dataset} The MNIST dataset consists of a collection of images of 10 handwritten digits; each sample is a $28 \times 28$ gray and sparse  image. We consider $300$ randomly selected digit images, resulting in a matrix $X$ of dimensions $784 \times 300$, where each column of the matrix contains the vectorized image of one digit. 
Then, we add sparse noise to the data by randomly flipping each white pixel to black and each gray pixel to white with a probability $p$, thereby obtaining a noisy version of the dataset $\Bar{X}$.
For this instance, we fix the regularization parameter $\lambda$ in the wL1-NMF to 1, that is, L1-NMF. We compare different NMF models in Section~\ref{sub:mnist_NMF_gen}, and we focus on the performance of the algorithms solving L1-NMF in Section~\ref{sub:mnist_l1NMF}.  We stop each algorithm after 90 seconds or if the difference in the relative error between two consecutive iterations is smaller than $10^{-6}$. 
The rank of the factorization is fixed to $r=50$.

\subsubsection{Comparison of NMF models}
\label{sub:mnist_NMF_gen}
Given $(W^*,H^*)$, an NMF computed by an algorithm, we measure the quality of the factorization using the relative error in Frobenius norm with the ground truth data, that is, $\lVert X-W^*H^* \rVert_F/\lVert X \rVert_F$. We specify that it is beyond the purposes of this work to compare our method to standard denoising techniques, which, in general, do not require the reconstruction to be low-rank.
Figure~\ref{fig:comp_NMF_models_MNIST} displays the relative error for various NMF models, 
for different values of the noise probability $p$. It shows that when the noise level is small ($p \leq 0.04$), KL-NMF, L2-NMF, and L21-NMF provide lower relative errors than L1-NMF, although comparable. 
However, as the noise level increases ($p \geq 0.08$), L1-NMF achieves the best performance. 
\begin{figure}
    \centering
    \resizebox{0.75\linewidth}{!}{ 
%
%
\definecolor{mycolor1}{rgb}{0.46600,0.67400,0.18800}%
\definecolor{mycolor2}{rgb}{0.49400,0.18400,0.55600}%
\definecolor{mycolor3}{rgb}{0.30100,0.74500,0.93300}%
\definecolor{mycolor4}{rgb}{0.93000,0.69000,0.13000}%
\begin{tikzpicture}

\begin{axis}[%
width=8in,
height=3.6in,
at={(0.733in,0.717in)},
scale only axis,
unbounded coords=jump,
font=\huge,
scaled y ticks=false,
    ticklabel style={
        /pgf/number format/fixed,
        /pgf/number format/precision=2
    },
xmin=-0.00725388395662747,
xmax=0.209844561639227,
xtick={   0, 0.04, 0.08, 0.12, 0.16,  0.2},
xlabel style={font=\color{white!15!black}},
xlabel={\huge Noise level ($p$)},
ymin=0.301647047123488,
ymax=1.3647058706529,
ylabel style={font=\color{white!15!black}},
ylabel={\huge Rel.\ err.\ ground truth},
axis background/.style={fill=white},
legend style={at={(0.02,0.58)}, anchor=south west, legend cell align=left, align=left, draw=white!15!black}
]
\addplot [color=red, line width=3.5pt, mark size=7pt, mark=o, mark options={solid, red,,line width=2.5pt}]
  table[row sep=crcr]{%
0	0.463938988879626\\
0.04	0.479124905191728\\
0.08	0.504456862974566\\
0.12	0.541133928421559\\
0.16	0.596806292046503\\
0.2	0.692836597284415\\
};
\addlegendentry{L1-NMF}

\addplot [color=mycolor1, line width=3.5pt, mark size=8pt, mark=diamond, mark options={solid, mycolor1,line width=2.5pt}]
  table[row sep=crcr]{%
0	0.377894274083328\\
0.04	0.442572833249210\\
0.08	0.570810121703690\\
0.12	0.699246833770925\\
0.16	0.815292021664598\\
0.2	0.923262049562748\\
};
\addlegendentry{L2-NMF}

\addplot [color=mycolor2, dashed, line width=5pt, mark size=8.5pt, mark=triangle, mark options={solid, mycolor2,line width=2.5pt}]
  table[row sep=crcr]{%
0	0.392533845650144\\
0.04	0.474994630721814\\
0.08	0.589967425719929\\
0.12	0.711260715970360\\
0.16	0.825521971753015\\
0.2	0.929886181700145\\
};
\addlegendentry{L21-NMF}

\addplot [color=mycolor3, dashdotted, line width=4.5pt, mark size=9.0pt, mark=x, mark options={solid, mycolor3,line width=2.5pt}]
  table[row sep=crcr]{%
0	0.424382181850498\\
0.04	0.561054740555926\\
0.08	0.667690659045612\\
0.12	0.772247074693669\\
0.16	0.875105179094526\\
0.2	0.975588202360953\\
};
\addlegendentry{KL-NMF}

\addplot [color=mycolor4, dotted, line width=4.5pt, mark size=6.5pt, mark=square, mark options={solid, mycolor4,line width=2.5pt}]
  table[row sep=crcr]{%
0	nan\\
0.04	0.583370689260315\\
0.08	0.824307759938737\\
0.12	1.00964329400092\\
0.16	1.16548022699053\\
0.2	1.30310127023990\\
};
\addlegendentry{Noisy data}

\end{axis}

\begin{axis}[%
width=5.833in,
height=4.375in,
at={(0in,0in)},
scale only axis,
xmin=0,
xmax=1,
ymin=0,
ymax=1,
axis line style={draw=none},
ticks=none,
axis x line*=bottom,
axis y line*=left
]
\end{axis}
\end{tikzpicture}%
}
 \caption{Comparison of the robustness of various NMF models on the MNIST dataset depending on the noise level. Relative error w.r.t.\ ground truth. Noisy data reports $\lVert X- \bar X \rVert_F/\lVert X \rVert_F$. 
 }
    \label{fig:comp_NMF_models_MNIST}
\end{figure}

This result is confirmed by the visual representations in Figure~\ref{fig:visual_rec}, which show examples of reconstructed digits, which are the columns of the NMFs $W^*H^*$ of $\Bar{X}$, reshaped as $28\times 28$ images.  
 L1-NMF is the only model able to recover the white background of the digits and provides much more accurate approximations, especially for higher noise levels.  
\begin{figure}[ht!]
\begin{center}
 \begin{minipage}[h]{.329\linewidth}
    \includegraphics[height=2.4cm]{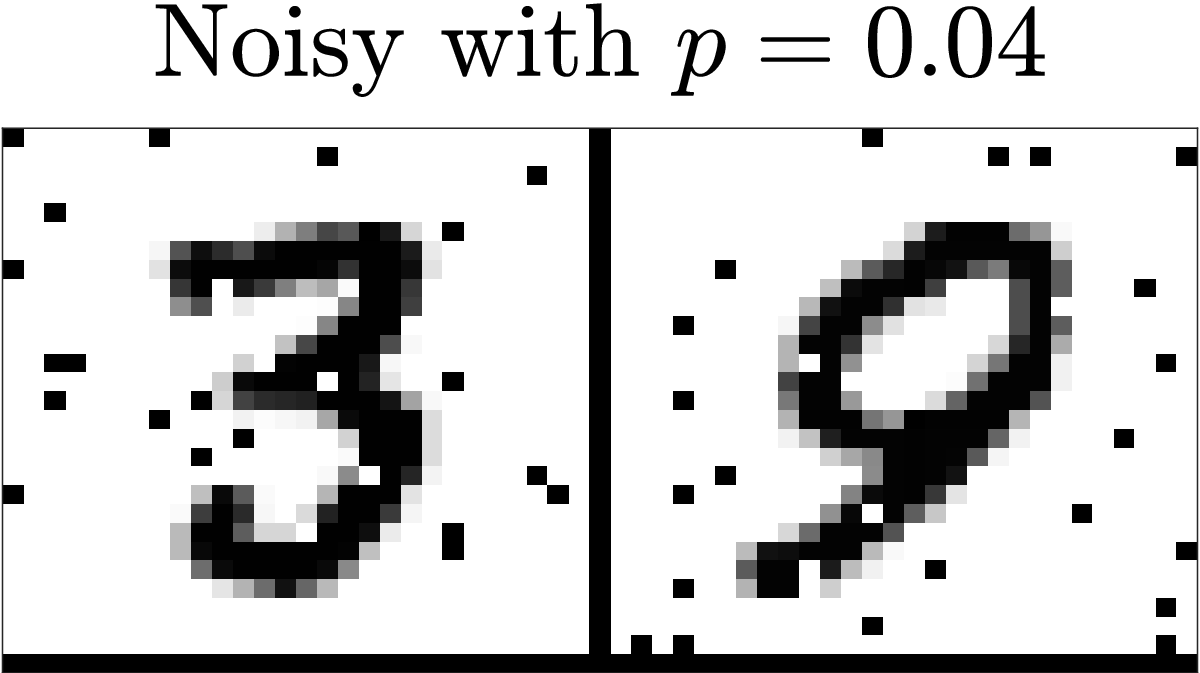}
\end{minipage}
\begin{minipage}[h]{.329\linewidth}
    \includegraphics[height=2.4cm]{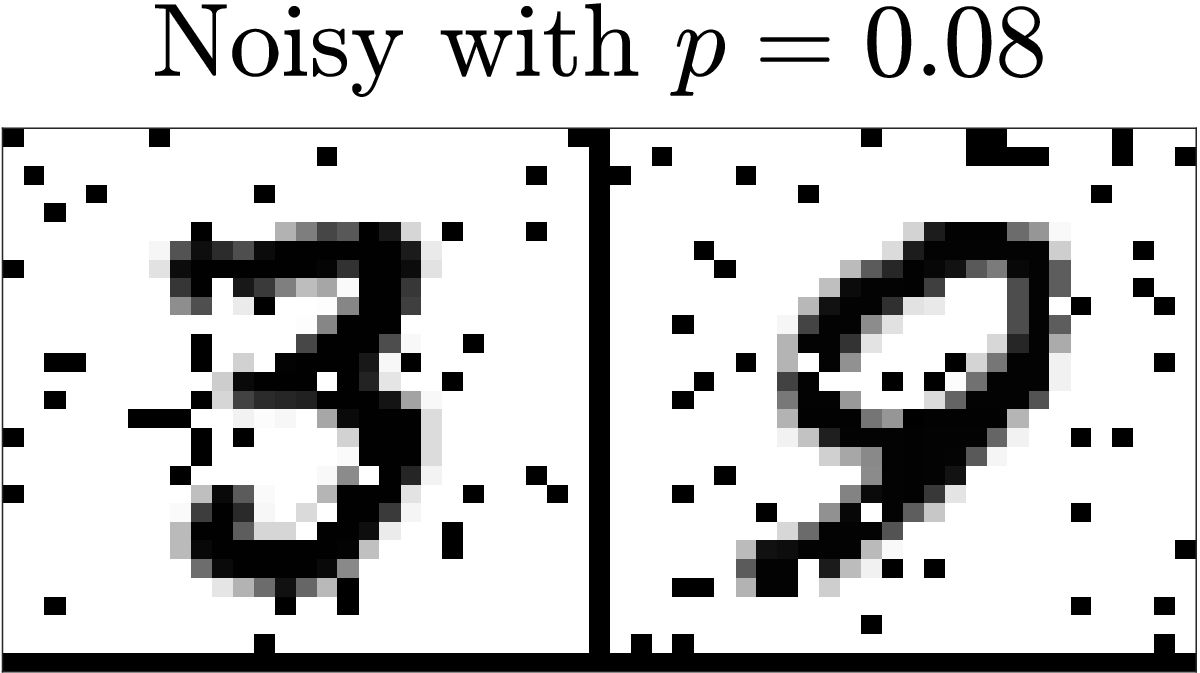}
\end{minipage}
\begin{minipage}[h]{.329\linewidth}
    \includegraphics[height=2.4cm]{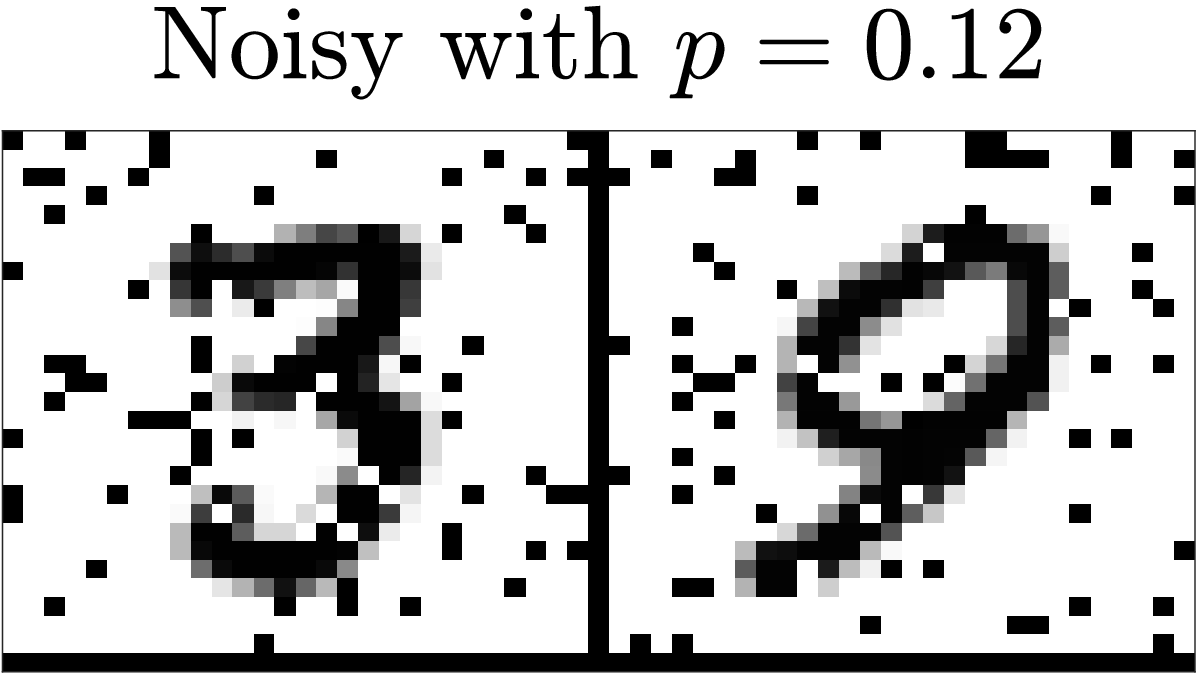}
\end{minipage}
\vspace{0.5cm}

 \begin{minipage}[h]{.329\linewidth}
    \includegraphics[height=2.4cm]{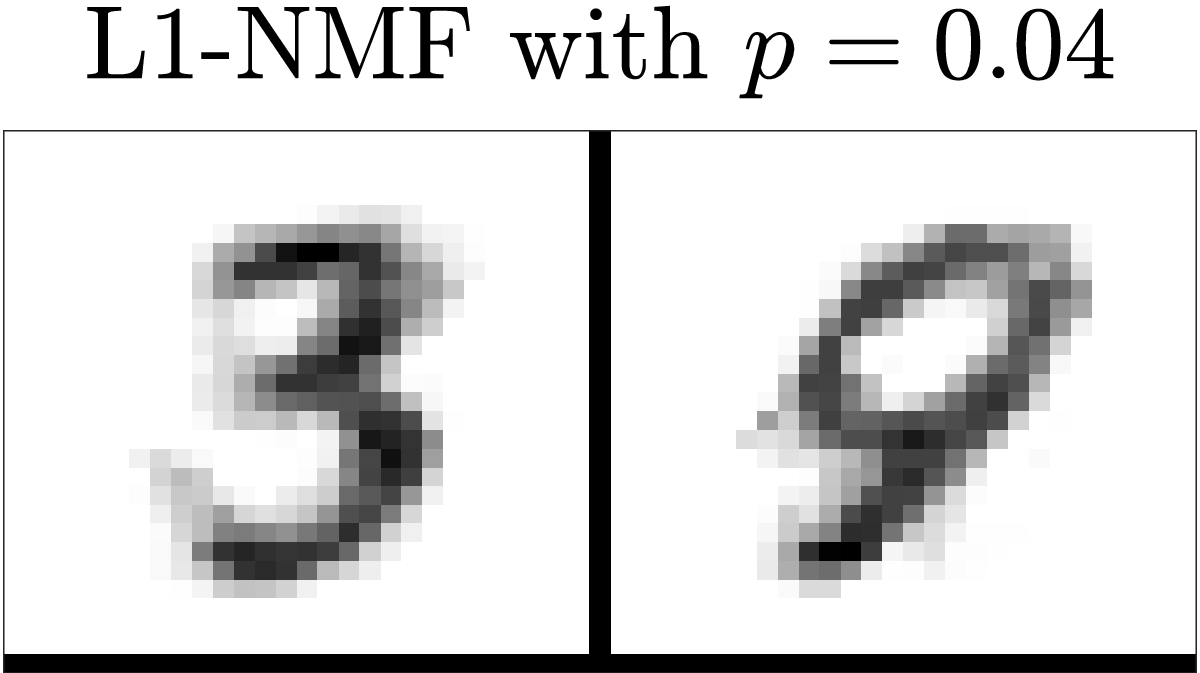}
\end{minipage}
\begin{minipage}[h]{.329\linewidth}
    \includegraphics[height=2.4cm]{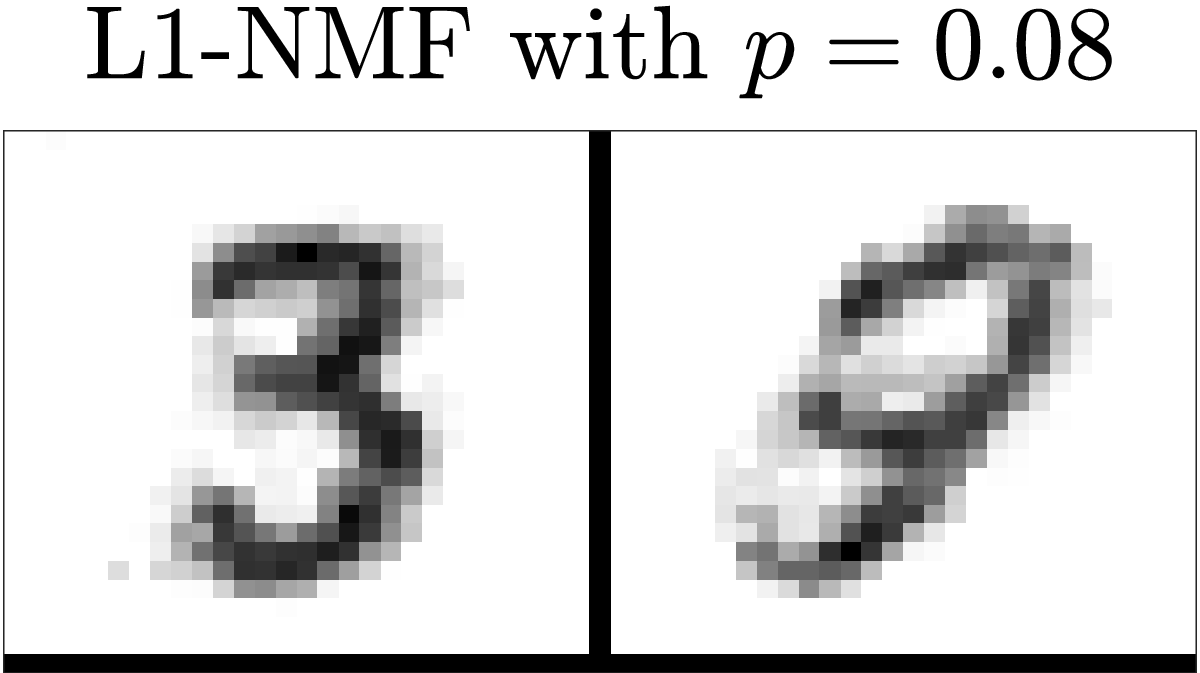}
\end{minipage}
\begin{minipage}[h]{.329\linewidth}
    \includegraphics[height=2.4cm]{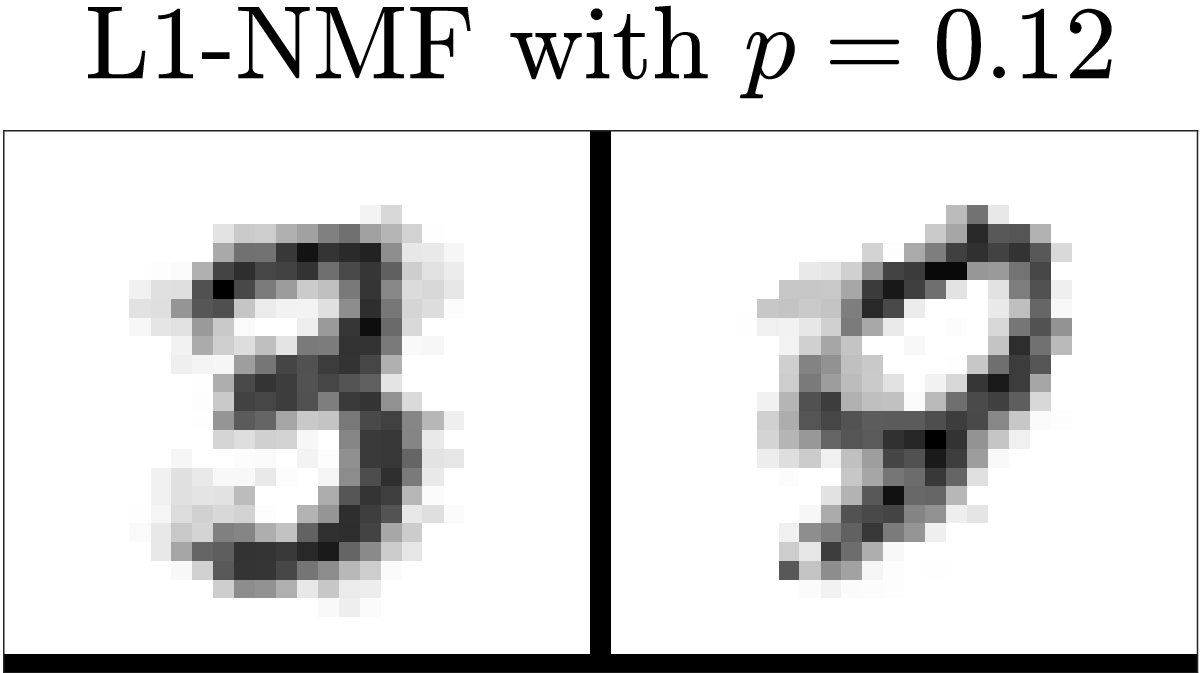}
\end{minipage}
\vspace{0.5cm}

  \begin{minipage}[h]{.329\linewidth}
    \includegraphics[height=2.4cm]{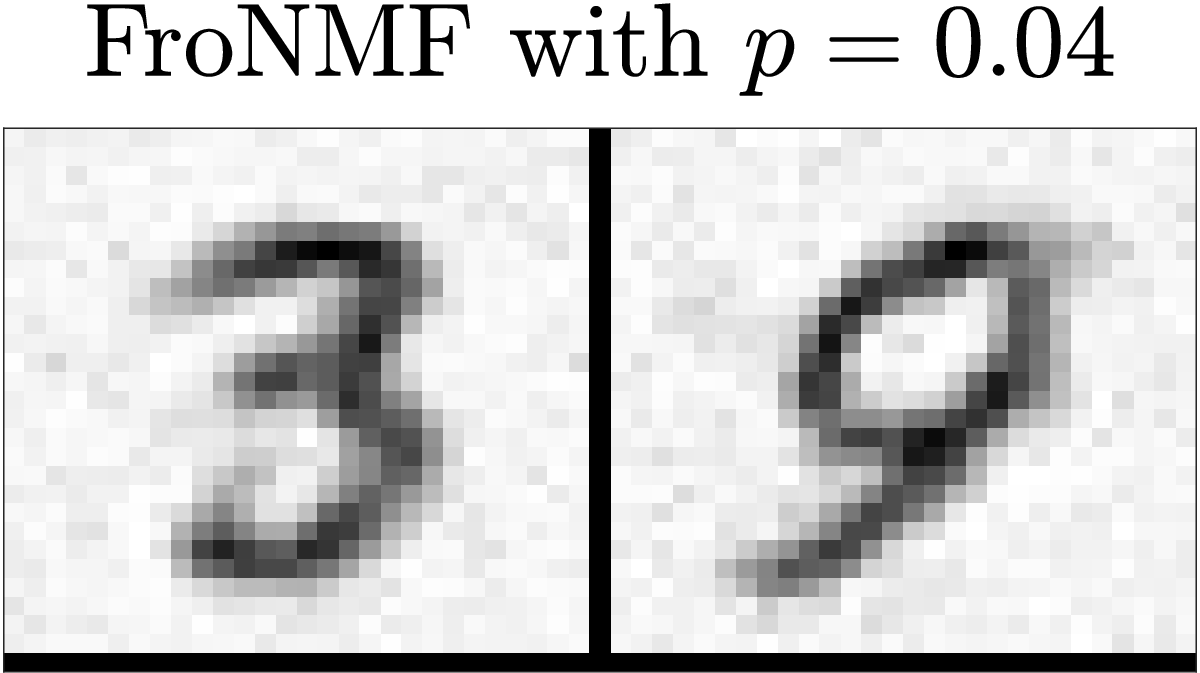}
\end{minipage}
\begin{minipage}[h]{.329\linewidth}
    \includegraphics[height=2.4cm]{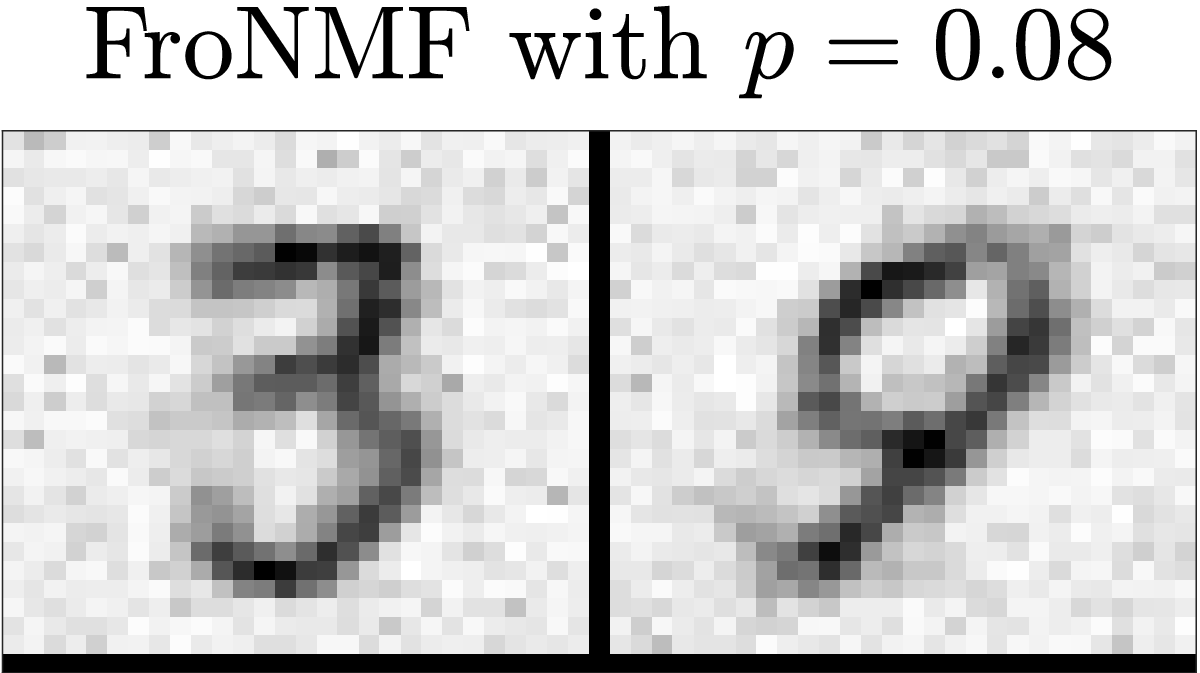}
\end{minipage}
\begin{minipage}[h]{.329\linewidth}
    \includegraphics[height=2.4cm]{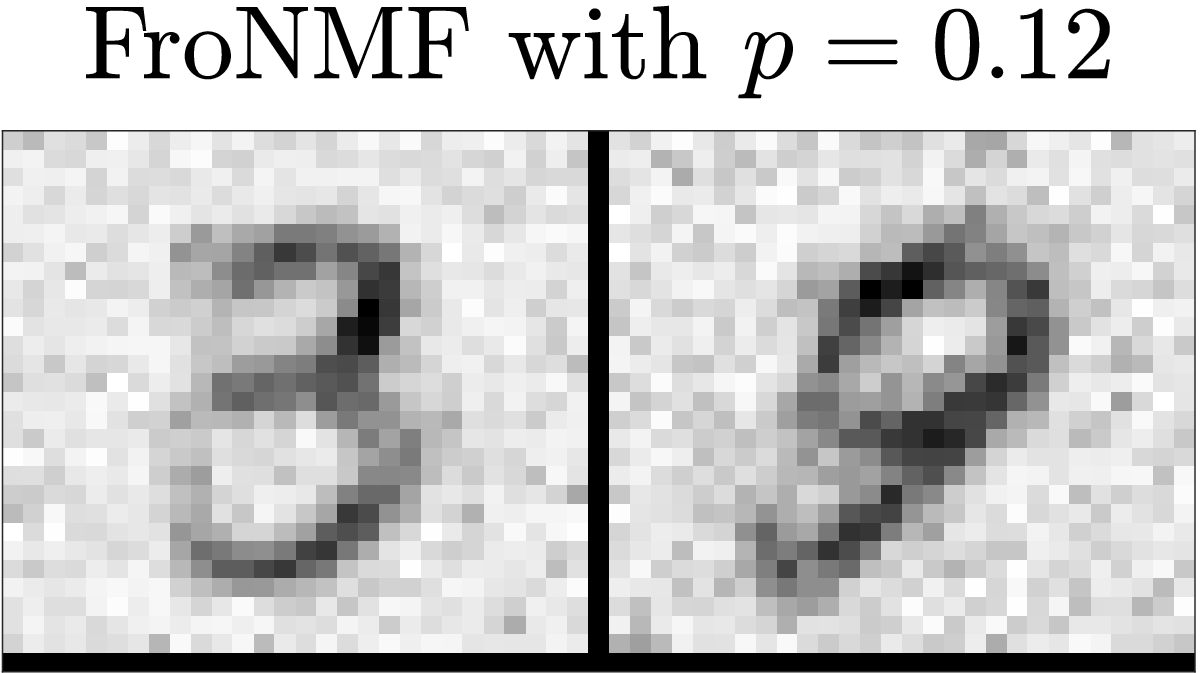}
\end{minipage}
\vspace{0.5cm}

 \begin{minipage}[h]{.329\linewidth}
    \includegraphics[height=2.4cm]{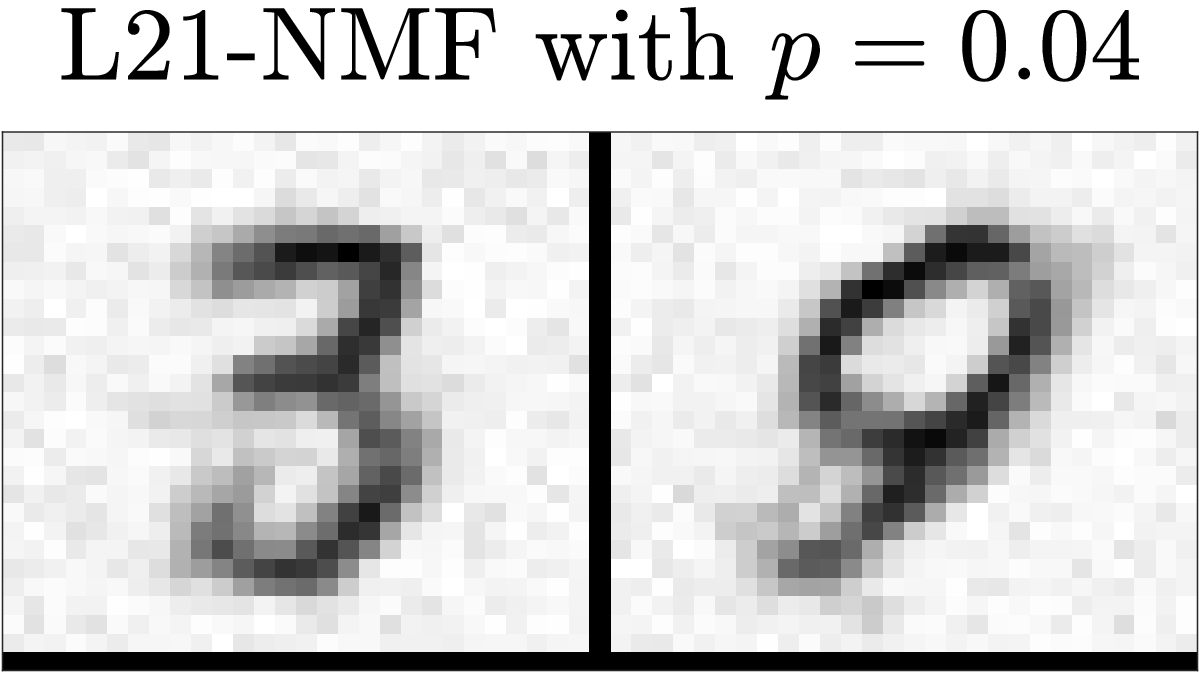}
\end{minipage}
\begin{minipage}[h]{.329\linewidth}
    \includegraphics[height=2.4cm]{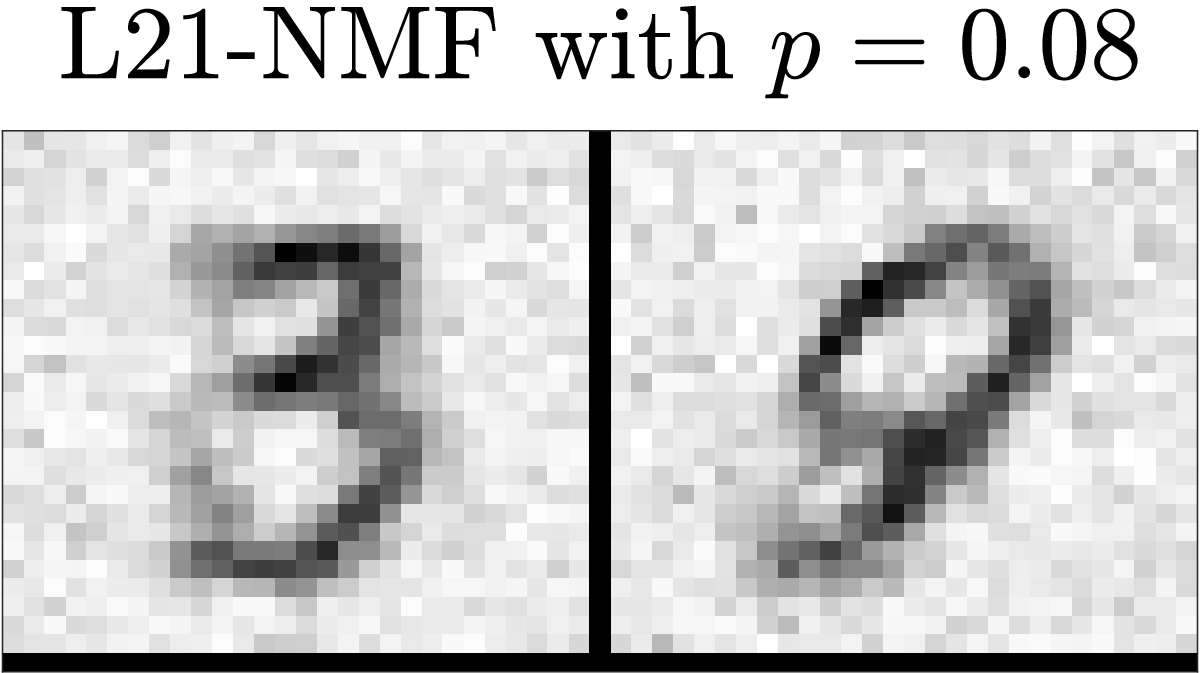}
\end{minipage}
\begin{minipage}[h]{.329\linewidth}
    \includegraphics[height=2.4cm]{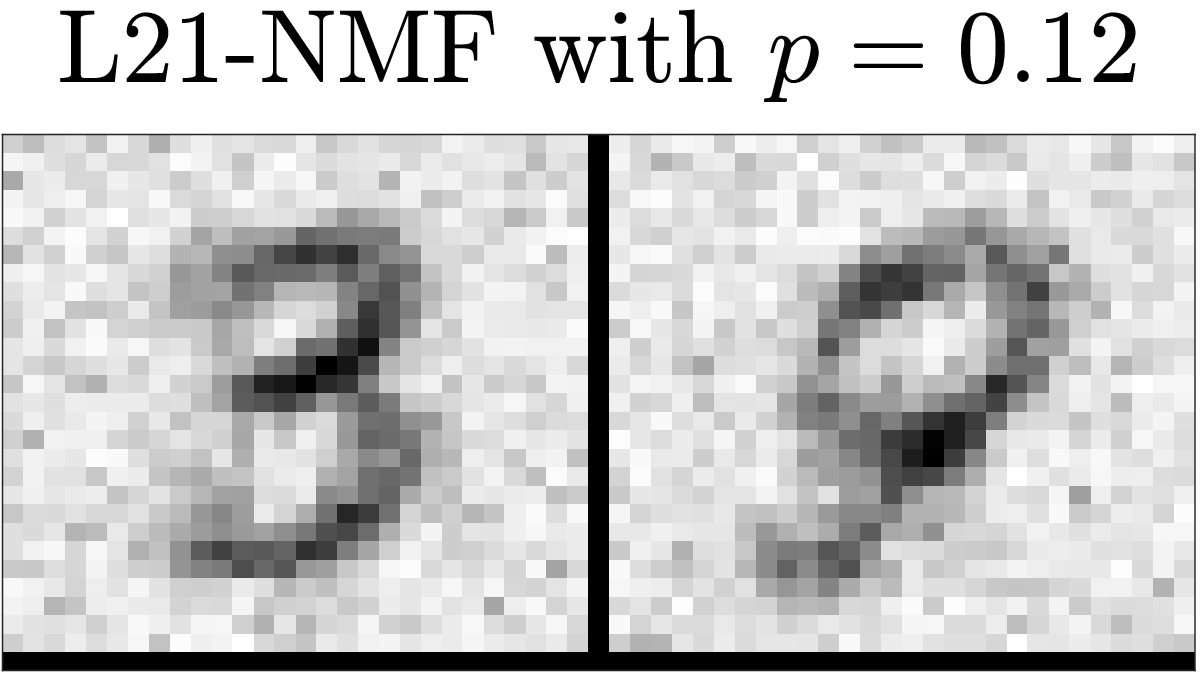}
\end{minipage}
\vspace{0.5cm}

 \begin{minipage}[h]{.329\linewidth}
    \includegraphics[height=2.4cm]{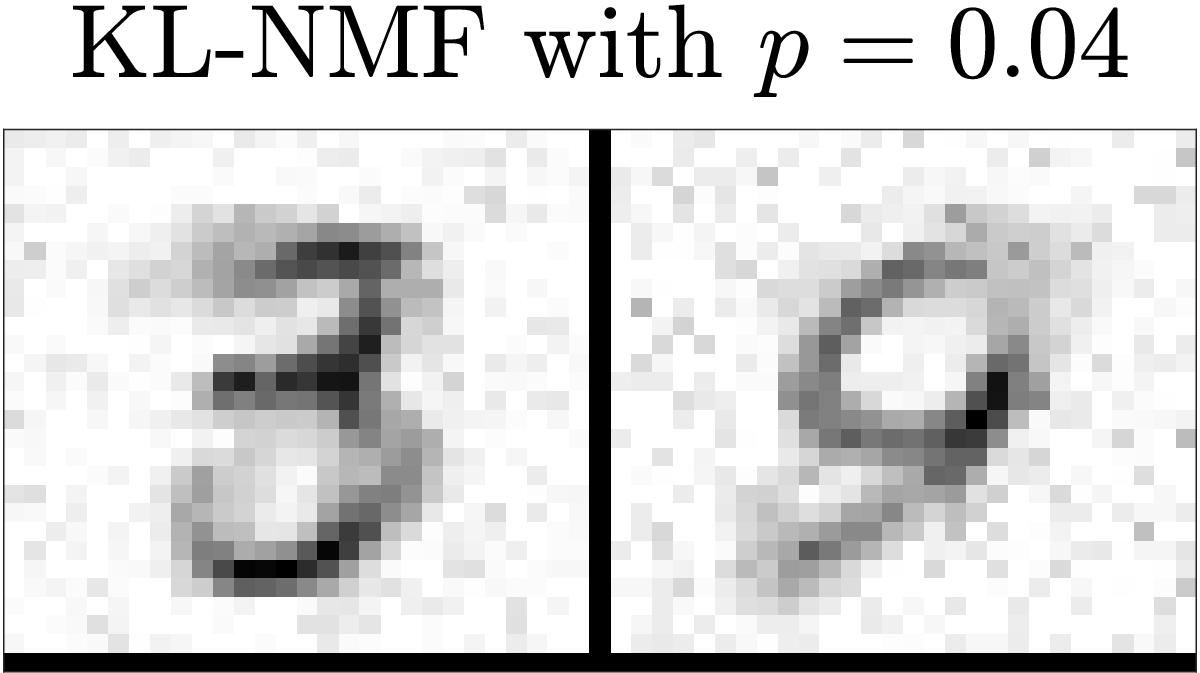}
\end{minipage}
\begin{minipage}[h]{.329\linewidth}
    \includegraphics[height=2.4cm]{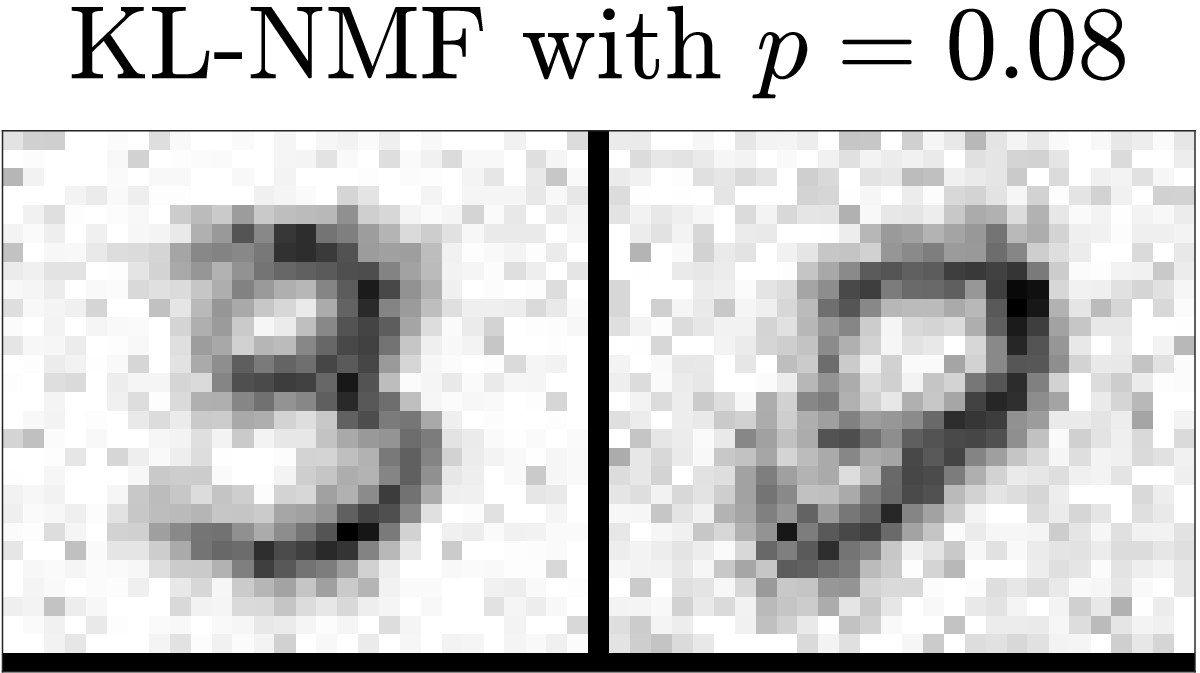}
\end{minipage}
\begin{minipage}[h]{.329\linewidth}
    \includegraphics[height=2.4cm]{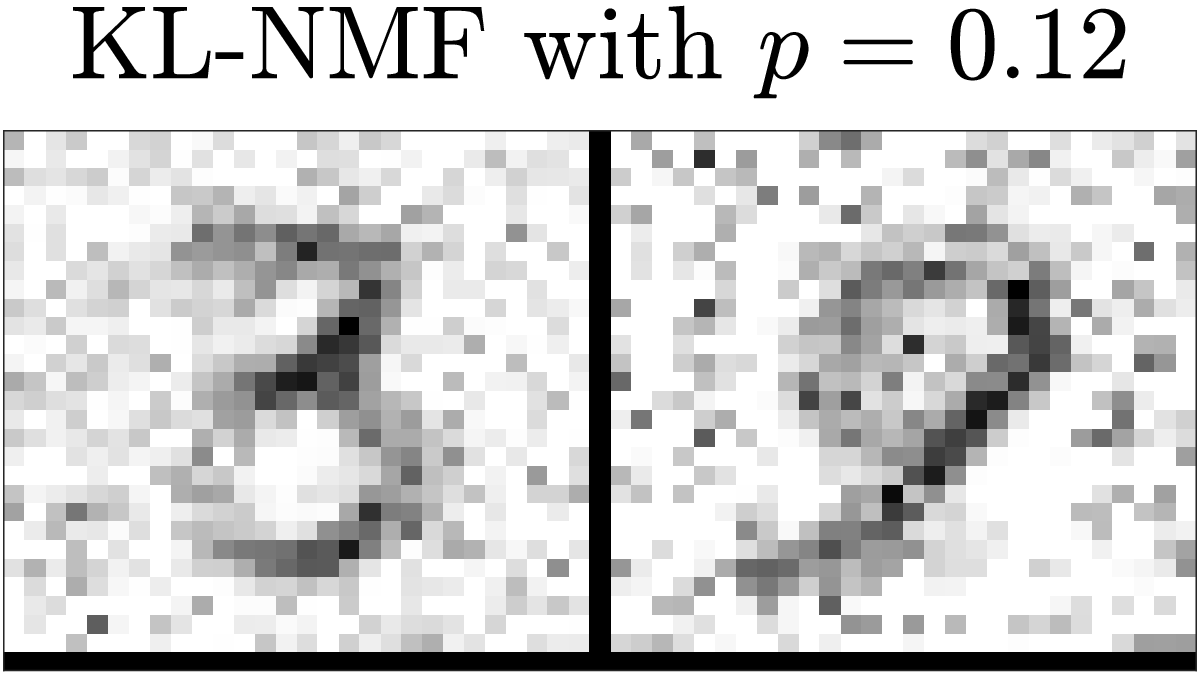}
\end{minipage}
\end{center}
\caption{Examples of the low-rank representations of digits using different NMF models on the MNIST dataset.}
\label{fig:visual_rec}
\end{figure}

\subsubsection{Comparison of L1-NMF algorithms}
\label{sub:mnist_l1NMF}
We now investigate the performance of different algorithms on the same MNIST dataset to solve L1-NMF in terms of final relative error, $\lVert \Bar{X} - W^* H^* \rVert_1 / \lVert \Bar{X} \rVert_1$. 
Table~\ref{tab:L1_MNIST_comp} reports the relative errors and the number of iterations performed by the different algorithms.  
 \begin{table}[h!]
\centering
\begin{tabular}{ c || c | c | c | c | c || c | c | c } 
 \toprule 
Noise &  & NS \cite{cao2023manhattan} & SUB \cite{rahimi2024projected} &  \name\ (Alg.~\ref{alg:sparse_l1_nmf})  &  CD \cite{ke2003robust}   & sparsity & $\sigma$ & gain  \\
  \midrule
     $0\%$   \multirow{5}{*}{}& res &   0.453 & 0.448 & \textbf{0.428}  & 0.437   & 81$\%$ & 6.2 & 5.8\\
  &  iter   & 183 & $7.0 \cdot 10^3$ & 29 &  5 & & &\\
   \hline
  $4\%$  \multirow{5}{*}{}& res  & 0.60 & 0.593 & \textbf{0.572} & 0.581 & 79$\%$ & 5.4 & 4.8\\
  
   &   iter   & 177 & $6.8 \cdot 10^3$ & 24 &  5 & & & \\
   \hline
  $8\%$   \multirow{5}{*}{}& res  & 0.710& 0.703 &  \textbf{0.675}  & 0.685 & 76$\%$ & 4.8 & 4.4\\

   & iter   & 178 & $7.0 \cdot 10^3$ & 22 &  5  & & & \\

  \hline
 $12\%$  \multirow{5}{*}{}& res  & 0.789& 0.783 &  \textbf{0.750}  & 0.760 & 74$\%$ & 4.3 & 4.0 \\
 
  &   iter   & 166 & $6.8 \cdot 10^3$ & 20 &  5 & & &\\

   \hline
  $16\%$  \multirow{5}{*}{}& res  & 0.840& 0.837 &  \textbf{0.804}  & 0.813 & 71$\%$ & 3.9 & 3.8\\
 
  &   iter   & 176 & $6.7 \cdot 10^3$ & 19 &  5 & &  &\\ 
 \bottomrule 
\end{tabular}
\caption{Comparison of L1-NMF algorithms on the MNIST data with different levels of noise. We report the sparsity of the noisy data (sparsity), the theoretical factor of compression of \name\ over CD ($\sigma$ defined in~\eqref{eq:comp_fact}), the practical gain (gain) of \name\ over CD, the final relative residual (res), and number of iterations performed (iter).}  
\label{tab:L1_MNIST_comp}
\end{table} 
 \name\ achieves the lowest residual for every level of noise, followed by the original CD algorithm, which is approximately 3 to 5 times slower (depending on the sparsity of the matrix $\bar X$, which depends on $p$). 
  NS and SUB have a consistently smaller computational cost per iteration; however, they cannot reach the same accuracy as the CD-based  approaches.

\subsection{Matrix completion with false zeros}

Let us now consider a nonnegative low-rank matrix completion problem where the missing entries correspond to zeros in the observed matrix. We are interested in a setting where some missing values are due to anomalies or errors in the measurements, e.g., as in imaging mass spectrometry~\cite{moens2025preserving,gonzalez2023nectar}, see Section~\ref{sec:w_ell_1_nmf}. 
We generate synthetic data where the missing values are of two types: missing completely at random or missing because they are below a fixed threshold.   
This represents a setting where some of the zeros are \textit{false zeros}, corresponding to low-intensity values. 
We aim to show that, in this application,  wL1-NMF performs better than the standard L1-NMF and low-rank nonnegative matrix completion, corresponding to the choices $\lambda=1$ and $\lambda=0$, respectively. 

Specifically, we generate two matrices $\hat{W} \in \mathbb{R}_+^{m \times r} $ and $\hat{H} \in \mathbb{R}_+^{r \times n}$ with $m=100$, $n=50$, and $r=20$ from a uniform distribution in $(0,1)$, and we form the ground truth low-rank matrix $\hat{W} \hat{H}$, which is dense. 
We then add a noise matrix $N$ to the ground truth data, whose entries follow a zero mean Laplace distribution with diversity $\sigma=0.1$, to obtain the noisy matrix ${X}=\Hat{W} \Hat{H} + N$. 
Finally, we replace a proportion $q_1 \in [0,1)$ of the entries of $X$ with zero.
We consider two types of zeros: a fraction $q_2$ of type I and a fraction $1-q_2$ of type II, with $q_2 \in [0,1]$: 
\begin{enumerate}
    \item type I sets the smallest entries of $X$ to zero, for a total of $q_1 q_2 mn$ entries. These entries can be seen as small entries thresholded to zero. We call them false zeros. 

    \item type II sets randomly selected entries of $X$ to zero, for a total of $q_1 (1-q_2) mn$ entries. We call them missing entries, as they do not bring any useful information since they were picked at random.  
\end{enumerate}
Hence for $q_2 = 0$, all zeros are actually missing and wL1-NMF with $\lambda = 0$ should perform better than $\lambda = 1$ (L1-NMF). On the contrary, for $q_2 > 0$, we expect wL1-NMF with $\lambda > 0$ to perform better since some of the zeros have been obtained by thresholding small values to zero.  

 Figure~\ref{fig:lambda_mc} shows the final relative errors with the original low-rank matrix, $\lVert \Hat{W} \Hat{H}-W^*H^* \rVert_F/\lVert \Hat{W} \Hat{H} \rVert_F$, 
 where $(W^*,H^*)$ is the  wL1-NMF solution for different values of $\lambda$, and for different combinations of percentages $q_1$ of zeros in the data and fractions $q_2$ of false zeros. 
\begin{figure}[ht!] 
\begin{center}
     \begin{minipage}[h!]{\linewidth}
     \begin{center}
         \resizebox{0.9\linewidth}{!}{\begin{tikzpicture}

\pgfplotsset{
    myplot/.style={
        xmin=-0.3, xmax=6.3,
        xtick={0,1,2,3,4,5,6},
        xticklabels={0,0.01,0.03,0.05,0.07,0.1,1},
        mark options={solid, fill=none, line width=1.2pt},
        grid=both
    }
}

\begin{groupplot}[
    group style={
        group size=2 by 2,
        horizontal sep=2cm,
        vertical sep=2cm
    },
    width=7cm,
    height=6.5cm,
    myplot,
    xlabel={$\lambda$},
    ylabel={Relative error},
    scaled y ticks=false,
    ticklabel style={
        /pgf/number format/fixed,
        /pgf/number format/precision=2
    },
]

\nextgroupplot[
    title={$q_1=40\%$ of zeros},
    ymin=0.05, ymax=0.2,
    legend to name=legendaComune,
    legend style={
        font=\fontsize{12pt}{14pt}\selectfont,
        mark size=3pt,
        line width=1.2pt,
        draw=black,
        fill=white,
        legend columns=-1,
        column sep=1em,
        mark options={solid, fill=none, line width=2pt},
    }
]

\addplot[blue, mark=o, mark size=4pt, line width=2pt] table[x=x,y=yA] {Tikz/Rec_system1_0_4.dat};
\addplot[red, mark=square, dotted, mark size=4pt, line width=2.5pt] table[x=x,y=yB] {Tikz/Rec_system1_0_4.dat};
\addplot[green!60!black, mark=triangle, dashdotted, mark size=5pt, line width=2.5pt] table[x=x,y=yC] {Tikz/Rec_system1_0_4.dat};
\addplot[orange, mark=diamond, dashed, mark size=5pt, line width=2.5pt] table[x=x,y=yD] {Tikz/Rec_system1_0_4.dat};
\addplot[purple!80!black, mark=x, mark size=5.5pt, line width=2pt] table[x=x,y=yE] {Tikz/Rec_system1_0_4.dat};

\addlegendentry{$q_2=100\%$}
\addlegendentry{$q_2=75\%$}
\addlegendentry{$q_2=50\%$}
\addlegendentry{$q_2=25\%$}
\addlegendentry{$q_2=0\%$}

\nextgroupplot[
    title={$q_1=50\%$ of zeros},
    ymin=0.05, ymax=0.3
]

\addplot[blue, mark=o, mark size=4pt, line width=2pt] table[x=x,y=yA] {Tikz/Rec_system1_0_5.dat};
\addplot[red, mark=square, dotted, mark size=4pt, line width=2.5pt] table[x=x,y=yB] {Tikz/Rec_system1_0_5.dat};
\addplot[green!60!black, mark=triangle, dashdotted, mark size=5pt, line width=2.5pt] table[x=x,y=yC] {Tikz/Rec_system1_0_5.dat};
\addplot[orange, mark=diamond, dashed, mark size=5pt, line width=2.5pt] table[x=x,y=yD] {Tikz/Rec_system1_0_5.dat};
\addplot[purple!80!black, mark=x, mark size=5.5pt, line width=2pt] table[x=x,y=yE] {Tikz/Rec_system1_0_5.dat};

\nextgroupplot[
    title={$q_1=60\%$ of zeros},
    ymin=0.05, ymax=0.4
]

\addplot[blue, mark=o, mark size=4pt, line width=2pt] table[x=x,y=yA] {Tikz/Rec_system1_0_6.dat};
\addplot[red, mark=square, dotted, mark size=4pt, line width=2.5pt] table[x=x,y=yB] {Tikz/Rec_system1_0_6.dat};
\addplot[green!60!black, mark=triangle, dashdotted, mark size=5pt, line width=2.5pt] table[x=x,y=yC] {Tikz/Rec_system1_0_6.dat};
\addplot[orange, mark=diamond, dashed, mark size=5pt, line width=2.5pt] table[x=x,y=yD] {Tikz/Rec_system1_0_6.dat};
\addplot[purple!80!black, mark=x, mark size=5.5pt, line width=2pt] table[x=x,y=yE] {Tikz/Rec_system1_0_6.dat};

\nextgroupplot[
    title={$q_1=70\%$ of zeros},
    ymin=0.1, ymax=0.6
]

\addplot[blue, mark=o, mark size=4pt, line width=2pt] table[x=x,y=yA] {Tikz/Rec_System1_0_7.dat};
\addplot[red, mark=square, dotted, mark size=4pt, line width=2.5pt] table[x=x,y=yB] {Tikz/Rec_System1_0_7.dat};
\addplot[green!60!black, mark=triangle, dashdotted, mark size=5pt, line width=2.5pt] table[x=x,y=yC] {Tikz/Rec_System1_0_7.dat};
\addplot[orange, mark=diamond, dashed, mark size=5pt, line width=2.5pt] table[x=x,y=yD] {Tikz/Rec_System1_0_7.dat};
\addplot[purple!80!black, mark=x, mark size=5.5pt, line width=2pt] table[x=x,y=yE] {Tikz/Rec_System1_0_7.dat};

\end{groupplot}

\node at ($(group c1r2.south)!0.5!(group c2r2.south) + (0,-1.5cm)$)
{\pgfplotslegendfromname{legendaComune}};

\end{tikzpicture}}
     \end{center}
     \end{minipage}    
       \caption{Relative error averaged over 10 runs between the noiseless data matrix and the wL1-NMFs for different values of~$\lambda$. 
       Each plot corresponds to different percentages $q_1$ of zeros, and each line represents a percentage $q_2$ of false zeros among the missing values. 
       }
       \label{fig:lambda_mc} 
        \end{center} 
\end{figure} 
 Figure~\ref{fig:lambda_mc} shows  that the choice of $\lambda=1$ (L1-NMF) yields worse results than all the other values of $\lambda$. 
 Moreover, as expected, $\lambda=0$ is a good choice when there are no false zeros among the zero entries (that is, $q_2 = 0$). 
 However, small values of $\lambda$ perform well even when $q_2 =0$, while they perform even better than $\lambda = 0$ for a high level of missing entries (namely, $q_1 = 0.7$, $q_2 = 0$). This is because the second term in  wL1-NMF acts as a regularizer which is well-known to improve the accuracy of low-rank factorization-based recommender systems when a large portion of the data is missing, avoiding overfitting~\cite{koren2009matrix}. Note also that this regularizer would be particularly meaningful  when data is not missing at random and missing entries may correspond to a negative preference~\cite{wang2018modeling}.  It would be interesting to use wL1-NMF in this context, which would be more robust than standard least-squares-based approaches~\cite{koren2009matrix}; this is a topic of further research. 
 
 In the presence of false zeros ($q_2 > 0$), L1-NMF has significantly larger relative errors than wL1-NMF, while wL1-NMF with $\lambda = 0$ 
 gets significantly worse as the number of false zeros increases (that is, $q_2$). A good choice of $\lambda$ is therefore crucial to obtain good factorizations in the presence of false zeros. wL1-NMF does not appear to be too sensitive to this parameter, as values of $\lambda$  between 0.01 and 0.05 provide good results in all cases. 

\subsection{\revise{Imaging Mass Spectrometry (IMS) with outliers}}
\revise{IMS is an analytical technique whose aim is to map the spatial distribution of molecules within a tissue sample by measuring ion signals generated by molecular ionization at each spatial location. A complete mass spectrum is recorded at each pixel location, and specific mass-to-charge (m/z) values are extracted to create the pixel-by-$m/z$ matrix of observation $X$. Due to the high resolution of IMS instruments, the acquired data are typically very large. A standard practice in IMS is therefore to reduce the information during the acquisition time by discarding low-intensity values. Consequently, a zero entry may indicate either the genuine absence of an ion response at that spatial location (a true zero) or an ion signal that was too weak to be detected or retained during acquisition (a false zero).

One of the key objectives in IMS data analysis is to identify and spatially separate the different materials or tissue types present within the sample, based solely on their underlying chemical composition. NMF is a standard technique used to perform this analysis~\cite{siy2008matrix,bemis2015cardinal,paine2016whole,nijs2021mathematical}, as it has a natural physical interpretation for IMS matrices: each column of $W$ represents the spatial abundance map of an extracted material: the value $W_{i,k}$ reflects how much the material $k$ is active in pixel $i$. 
 
We assess the ability of wL1-NMF to address this task.
We consider the IMS dataset pig206, a publicly available MSI benchmark distributed with the Cardinal R/Bioconductor package. It consists of a whole-body sagittal section of a fetal pig, with 4,959 spatial pixels, each with a full mass spectrum recorded over 10,200 $m/z$ channels. 
We apply standard preprocessing techniques for IMS using the Cardinal R/Bioconductor package~\cite{bemis2015cardinal}, namely row normalization, Gaussian smoothing, and baseline correction.

We investigate a noisy scenario in which the data is perturbed by 30 localized artifacts. Each artifact is generated according to the following strategy: 
\begin{itemize}
    \item We randomly select the position of 30 pixels  among the rows of $X$. 

    \item For each selected pixel, we randomly select the position of 6 $m/z$ channels among the columns of $X$. 
    
    \item For each selected pixel, we also select the 8 adjacent pixels and we change the values of these 9 pixels to the same value, $c_{\text{out}}$, in the 6 selected $m/z$ channels. 
    The value of $c_{\text{out}}$ is picked uniformly at random in the interval $[0.8 \, \lVert  X\rVert_\infty  , \lVert X\rVert_\infty]$, where $\lVert X\rVert_\infty$ is the largest entry in $X$. 
\end{itemize}
This setting aims at emulating acquisition anomalies (such as matrix crystallization, particulate contamination or even error in data transmission) that typically produce strong, spatially confined ionization signals often referred to as hotspots~\cite{wullems2021fast}.
Moreover, we simulate the clipping phenomenon by retaining only the top 5$\%$ of the nonzero entries in the matrix and setting the remaining entries to zero. We remove any all-zero rows or columns created by this operation. We generate 10 independent corrupted versions of the dataset using the described procedure, obtaining 10 pixel-by-$m/z$ matrices with dimensions of approximately $5000 \times 2500$ (depending on the zero rows and columns removal), with an average sparsity of $91.2 \%$. 

Figure~\ref{fig:true_corr_pig} displays, on the left,  an optical scan of the pig fetus which serves only for visualization purposes. It also shows the total ion current images of the original (middle) and corrupted (right) images, where the artifacts are represented by the localized, high intensity, yellow regions.

}
\begin{figure}
    \centering
      \includegraphics[width=\linewidth]{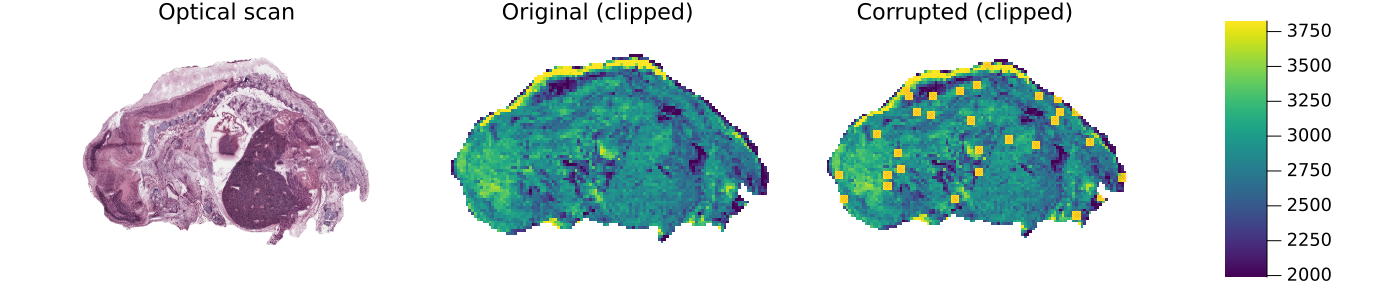}
 \caption{\revise{Optical scan (left) of the pig206 dataset from the Cardinal R/Bioconductor package highlighting main tissues in the pig fetus such as brain (left), heart (center), and liver (below the heart). 
 The middle and right images are total ion current images of the pig206 dataset, obtained by summing all columns of $X$, before (middle) and after (right) injection of 30 localized artifacts (yellow regions).} } 
    \label{fig:true_corr_pig}
\end{figure}

\revise{

The choice of the rank is related to the number of different structures we expect to extract. Our experience shows that $r=8$ is a reasonable choice for the pig206 dataset which allows to separate the main structure in the pig fetal such as the brain, heart, and liver, and achieves a relative Frobenius error with respect to the uncorrupted data of $36 \%$ and $37 \%$ for wL1-NMF and FroNMF, respectively, see Table~\ref{tab:comp_pig}.
We compare wL1-NMF and FroNMF in their ability to spatially separate the different materials from the noisy data. We choose $\lambda=0.1$ after trying several choices of regularization parameters, and we include the results for other values of $\lambda$ in the supplementary material. We run the algorithms for at most 120 seconds and stop them if the difference in the relative error between two function values is below $10^{-6}$. 
We report in Figure~\ref{fig:IMS_factor_W} the plot of the factors of FroNMF and wL1-NMF as heatmaps for one of the 10 corrupted and clipped data, comparing it with noiseless version of the data with the clipping. Ideally, each factor should represent one or a combination of tissues. We observe on Figure~\ref{fig:IMS_factor_W} that both models exhibit similar results in the noiseless setting and they are both able to identify the main tissues,  such as the liver (factor 5), the brain (factor 8), the heart (factor 7), and the adipose layer (factor 3). Nevertheless, we can highlight some differences for example in factor 2 where WL1-NMF detect a mixed brain/liver abundance map while FroNMF catches background structures. In the corrupted scenario, wL1-NMF is almost unaffected by the presence of the outliers while FroNMF is significantly. In fact, three out of eight factors, namely factors 3, 7, and 8, detect outliers  instead of a biologically meaningful tissue. 

\begin{figure}
    \centering
    \includegraphics[width=\linewidth]{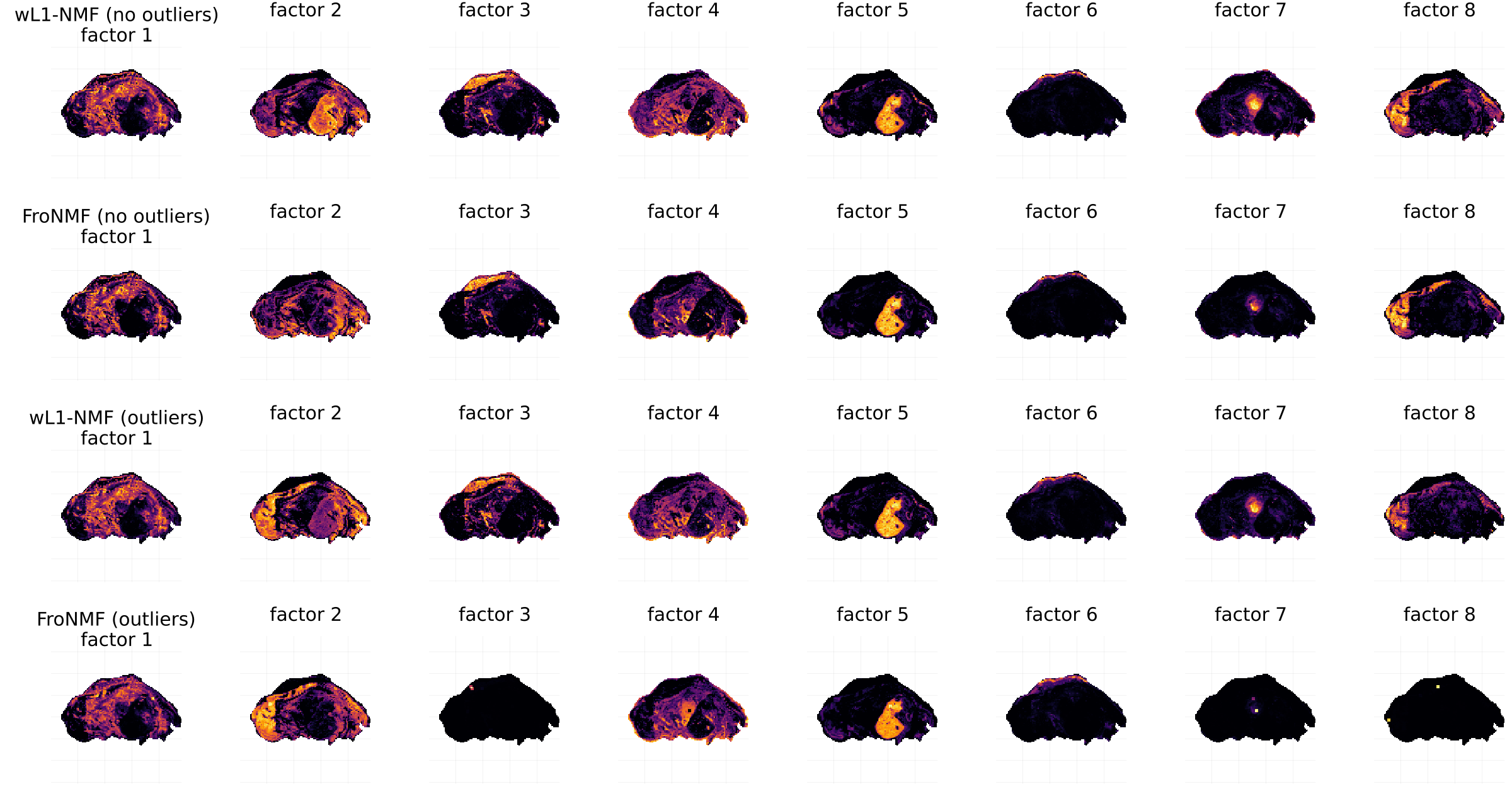}
    \caption{\revise{Spatial abundance maps obtained from the columns of the factor $W$ of $\rank$-8 wL1-NMF (first and third rows) and FroNMF (second and last rows) on the uncorrupted (no outliers) and corrupted (outliers) pig206 dataset. Interpretation of the factors: general background/tissue (factor 1), mixed brain/liver (factor 2), subcutaneous adipose layer (factor 3), connective tissue (factor 4), liver (factor 5), skin (factor 6), heart (factor 7), brain and spinal cord (factor 8). } }
    \label{fig:IMS_factor_W}
\end{figure}

Let $(\Hat{W},\Hat{H})$ be the approximation computed either by WL1-NMF or FroNMF, we compare in Table~\ref{tab:comp_pig}, the results provided by wL1-NMF and FroNMF in terms of the relative errors in Frobenius norm, that is $||X - \Hat{W}\Hat{H}||_F / \|X\|_F$) and in the L1 norm, that is $||X -\Hat{W}\Hat{H}||_1 / \|X\|_1$, where $X$ is uncorrupted. We consider both the uncorrupted and corrupted cases, namely $(\Hat{W},\Hat{H})$ is obtained factorizing the data without and with outliers, respectively. 
Moreover, to measure how much the injected outliers  affects each NMF factor, we evaluate the correlation between each column $\Hat{W}_{:,k}$ of $\Hat{W}$ and the binary vector of pixels positions $b$, containing a $1$ if the corresponding pixels is corrupted and zero elsewhere. This is defined as follows: 
\begin{equation}
    \rho_k=\frac{(\Hat{W}_{:,k}-\Hat{W}_{:,k}^\text{mean})^\top (b-b^\text{mean})}{\lVert \Hat{W}_{:,k}-\Hat{W}_{:,k}^\text{mean} \rVert_2 \lVert b-b^\text{mean}   \rVert_2} \in [-1,1],
    \label{eq:corr}
\end{equation}
where $\Hat{W}_{:,k}^\text{mean}$ and $b^\text{mean}$, represent the vectors containing all entries equal to the mean of $\Hat{W}_{:,k}$ and $b$, respectively. A value of $\rho_k$ close to 1 indicates that the $k$-th NMF factor is predominantly explained by the injected artifact, while a value close to 
0 or even negative means that the factor is unaffected by the corruption. We compute both the maximum ($\rho_{\max}=\max_k \rho_k)$  and the mean ($\rho_{\text{mean}}=1/r \sum_{k=1}^r \rho_k)$ correlation, which are reported in Table~\ref{tab:comp_pig}. 

\begin{table}[ht]
    \centering
    \begin{tabular}{ccccc}
        \hline
         & L1 error & Fro error & $\rho_{\max}$ & $\rho_{\text{mean}}$ \\
         \hline
        wL1-NMF (uncorrupted) & $\mathbf{0.489 \pm 0.002}$ & $\mathbf{0.359 \pm 0.006}$& / &  / \\
        FroNMF (uncorrupted)& $0.546 \pm 0.0005$ & $0.374 \pm 0.0005$ & / & / \\
        \hline
         \hline
        wL1-NMF (corrupted) & $\mathbf{0.507 \pm 0.018}$ & $0.533 \pm 0.161$& $\mathbf{0.210 \pm 0.122}$ &  $\mathbf{0.064 \pm 0.03}$ \\
        FroNMF (corrupted)& $0.580 \pm 0.005$ & $\mathbf{0.466 \pm 0.019}$ & $0.298 \pm 0.037$ &  $0.112 \pm 0.014$\\
        \hline
    \end{tabular}
    \caption{\revise{Comparison of wL1-NMF ($\lambda=0.1$) and FroNMF applied to the pig206 dataset without corruptions and with high intensity injected outliers. The results are averaged over 10 runs and show two metrics for assessing the quality of the reconstruction (L1 and Fro error) and the max ($\rho_{\max}$) and mean ($\rho_{\text{mean}}$) correlation computed from~\eqref{eq:corr}. Higher values of the quality metrics and lower values of the correlations correspond to the best performance and are highlighted in bold.} \label{tab:comp_pig}}
\end{table}

We observe from Table~\ref{tab:comp_pig} that wL1-NMF achieves a lower error in the uncorrupted setting, in both the error metrics. In the corrupted setting, the visual intuition from Figure~\ref{fig:IMS_factor_W} is confirmed by the higher value in both max and mean correlation for FroNMF with respect to wL1-NMF in Table~\ref{tab:comp_pig}. 
Moreover, wL1-NMF achieves a lower L1 reconstruction error with respect to the uncorrupted data, whereas FroNMF achieves the lowest Frobenius error.  
}

\subsection{Topic modeling}
We consider the Topic Detection Task (TDT2) dataset, which contains a collection of articles from various sources, collected from January to June of 1998, on $r = 30$ selected topics. It consists of a matrix $X \in \mathbb{R}^{m \times n}$ with $m =19528$ and $n = 9394$, where the $(i,j)$th entry is the number of occurrences of the word $i$ in document $j$.
Such datasets are typically very sparse since only a small portion of the words in the dictionary are used in each document. The percentage of zero entries in TDT2 is $99.37\%$. 
When computing the NMF of $X \approx WH$, the columns of $W$ represent topics, that is, bags of words that are used to reconstruct multiple documents (the columns of $X$), while the matrix $H$ tells us which topics are discussed in which document~\cite{lee1999learning}. 
We present this application for two reasons: 
\begin{enumerate}
    \item[(a)] This extremely sparse scenario is a typical situation where L1-NMF might provide too sparse solutions, and we show that wL1-NMF leads to significantly better  solutions. 

\item[(b)] The dimensions of the dataset cause the NS and SUB algorithms to run out of memory after few iterations and make the standard CD approach extremely slow (more than 4 hours per iteration, $\sigma$ is equal to 175 for this dataset) on the  computer we used for these experiments. 
Exploiting the sparsity of the data, \name\ has a reasonable cost per iteration (about 70 seconds). 
\end{enumerate}

It is beyond the scope of this paper to compare our algorithm to state-of-the-art techniques in topic modeling, and we compare our results to FroNMF.  
 The wL1-NMF factorizations are computed by 15 iterations of \name\, initialized with 3 iterations of HALS, while we use 30 iterations of HALS to compute FroNMF.
 We present the results for a single random initialization to be able to compare the recovered topics. 
 The computational cost of wL1-NMF is consistently larger than that of FroNMF. In fact, \name\ requires approximately 70 seconds, while HALS requires approximately 1 second per iteration. 
 We observed that, in such a sparse scenario, initializing the wL1-NMF model with HALS is crucial as it allows to obtain reasonably dense results for larger values of~$\lambda$. If random initialization is used, the $\lambda$ parameter needs to be smaller in order to prevent extremely sparse solutions.

Table~\ref{tab:topic_mod} reports the average number of nonzero entries per column in $W$, that is, the average number of words in each topic. It also displays the number of topics with less than 5 words, as we consider them less meaningful. 
\begin{table}[ht!]
\begin{center}
\begin{tabular}{c||c|c|c|c|c|c|c}
\hline
&  \multicolumn{6}{c|}{wL1-NMF} & FroNMF\\
 \midrule
$\lambda$ & 0 & 0.01 & 0.04 & 0.08 & 0.1 & 1.0 & / \\
\hline
Average words per topic & 12898 & 1780 & 417 & 170 & 123 & 4 & 5275  \\
\hline
topics less than 5 words & 0 & 0 & 0 & 0 & 1 & 24 & 0  \\
\hline
\end{tabular}
\caption{Impact of the parameter $\lambda$ on the 
wL1-NMF solution on the TDT2 data.} 
\label{tab:topic_mod}   
\end{center}
\end{table} 
As expected, the larger $\lambda$, the sparser the factor $W$ in wL1-NMF, that is, the fewer the number of words describing each topic. This is a clear example where L1-NMF ($\lambda=1$) produces too sparse solutions since 24 out of 30 topics are described by less than 5 words due to the original sparsity in the data. 
A good choice of the $\lambda$ parameter, such as $\lambda=0.08$, enhances the interpretability of the factorization by providing sparser topics. 
In fact, 
wL1-NMF 
removes approximately $97 \%$ of the words for each topic compared to FroNMF; from 5275 words on average for FroNMF, to 170 for wL1-NMF with $\lambda = 0.08$.  
A sparser factorization means a more straightforward interpretation and fewer mixed topics. 

Let us provide a qualitative comparison.  
Table~\ref{tab:topic_mod_qual} shows, for two topics, the ten most important words, that is, the words corresponding to the ten largest entries in each column of $W$. 
The other topics are provided in the supplementary material. 
Overall, the topics extracted by FroNMF and wL1-NMF with $\lambda=0.08$ are similar; thus, we were not able to identify which factorization provides the best topic extraction in general, and the results vary from topic to topic. 
\begin{table}[h!]
\begin{center}
\begin{tabular}{c|c|c||c|c|c}
\hline
  \multicolumn{2}{c|}{wL1-NMF} & FroNMF & \multicolumn{2}{c|}{wL1-NMF} & FroNMF\\
 
 \midrule
   0.01  & 0.08   &  /  &  0.01  &   0.08   &  /  \\
\hline
  \textcolor{red}{kaczynski} & gov. & gov. &  \textcolor{red}{european} & \textcolor{red}{court}  & super\\

   \textcolor{red}{algerian} & party  & party & \textcolor{red}{martin} & \textcolor{red}{supreme}  & bowl \\
   \textcolor{red}{algeria} & india  &  \textcolor{red}{algeria} & \textcolor{red}{tour} & \textcolor{black}{bowl}  & game \\
   \textcolor{red}{abacha} & bjp  & election & \textcolor{red}{documents} & \textcolor{black}{super}  & denver \\
  srinagar & election  & elections &  \textcolor{red}{algeria} & bulls & broncos \\
 \textcolor{red}{massacres} & vajpayee  &  \textcolor{red}{algerian} & \textcolor{red}{privilege} & chicago  & packers \\
 gandhi & hindu  & \textcolor{red}{killed} & \textcolor{red}{golf} & jazz & green \\
  \textcolor{red}{algiers} & minister & india &  \textcolor{red}{pga} & utah  & bay \\
 bjp & elections  &  political &\textcolor{red}{supreme} & jordan  &  san\\
 india & voting  & \textcolor{red}{islamic} & \textcolor{red}{court} & game  & football \\
\hline
 \multicolumn{3}{c||}{Topic 6: Vajpayee election} &  \multicolumn{3}{c}{Topic 17: Sport / Super bowl}\\
 \hline
\end{tabular}
\caption{Ten most important  words for two extracted topics (Topic 6 and Topic 17) by  FroNMF and wL1-NMF on the TDT2 dataset. 
In red, we highlight the words that do not have a clear connection with the topic that the majority of the other words refer to. 
}
\label{tab:topic_mod_qual}
\end{center}
\end{table}
For $\lambda = 0.01$, wL1-NMF provides too dense solutions, with words that have no clear connection with each other (Topic 6) or words that are not related to the topic (Topic 17). 
wL1-NMF with $\lambda = 0.08$ provides a more meaningful Topic 6, in our opinion, related to the 1998 election of Vajpayee of the BJP party as prime minister of India. Even though the same event is covered by FroNMF, there are some unrelated words, such as "Algeria" and "killed". Conversely, Topic 17 extracted by FroNMF discusses the American football game (Super Bowl), in which the Denver Broncos faced the Green Bay Packers.
By contrast, the topic extracted by wL1-NMF with $\lambda = 0.08$ also contains two words  (``court'', ``supreme'') that are not related to the topic and mixes in vocabulary from the NBA finals between the Chicago Bulls and the Utah Jazz.
It therefore provides a more mixed topic than FroNMF. 

\section{Conclusion}

In this paper, we considered L1-NMF that minimizes the component-wise L1 norm between the data matrix and its low-rank approximation. 
We first established the NP-hardness of L1-NMF even in the rank-1 case. Then, we showed that L1-NMF produces sparser solutions than L2-NMF for sparse data. This motivated us to introduce  weighted L1-NMF (wL1-NMF), suitable for sparse data corrupted by false zeros. 
We proposed a CD algorithm for wL1-NMF, 
namely \name, that scales linearly with the number of nonzeros in the data, up to a logarithmic factor. It is, to the best of our knowledge, the first L1-NMF algorithm that can handle large-scale sparse data. 
We performed extensive numerical experiments on synthetic and real datasets, showing that wL1-NMF outperforms other NMF models when the data is  affected by sparse noise or false zeros. Moreover, our new algorithm, \name, performs well against the state of the art for L1-NMF and reduces the computational time of the original CD algorithm consistently when the data is sparse. We finally showed that wL1-NMF can produce meaningful \revise{spatial separation of tissues in corrupted IMS images, and also comparable} results to FroNMF on a large and sparse document dataset (namely, TDT2) . 

\revise{
\section*{Acknowledgments}  We are grateful to the anonymous referees who
carefully read the manuscript; their feedback helped us improve our paper.
}

\bibliographystyle{spmpsci.bst}
\bibliography{bib.bib}

\newpage
\section*{Supplementary to topic modelling application}
This supplementary material provides the full results for the topic modeling experiments on the TDT2 data set. Table~\ref{tab:topic_mod_qual_sup} contains the extracted topics for the TDT2 dataset for wL1-NMF with $\lambda=0.08$ and for FroNMF.

\setlength{\tabcolsep}{2pt}
\begin{longtable}[H]{c|c||c|c||c|c}

\hline
wL1 & Fro & wL1 & Fro & wL1 & Fro \\
 \midrule
\endfirsthead

\hline
wL1 & Fro & wL1 & Fro & wL1 & Fro \\
 \midrule
\endhead

spkr & saddam & iraq & weapons & spkr & president \\
president & secretary & iraqi & iraq & president & clinton \\
clinton & hussein & team & inspectors & clinton & house \\
saddam & albright & baghdad & iraqi & china & white \\
iraq & iraqi & sanctions & baghdad & says & clintons \\
hussein & president & weapons & butler & chinese & allegations \\
house & military & iraqis & sites & government & jones \\
white & weapons & ritter & inspections & correspondent & sexual \\
albright & foreign & butler & iraqs & rights & washington \\
diplomatic & madeleine & inspectors & presidential & human & china \\

\hline
\multicolumn{2}{c||}{Topic 1} &
\multicolumn{2}{c||}{Topic 2} &
\multicolumn{2}{c}{Topic 3}\\
\hline
\hline

percent & percent & today & today & government & government \\
market & stock & checking & top & party & party \\
stock & points & stories & stories & india & algeria \\
points & market & top & checking & bjp & election \\
dollar & index & secretary & hour & election & elections \\
stocks & stocks & hour & voice & vajpayee & algerian \\
index & prices & says & tomorrow & hindu & killed \\
yen & investors & washington & peter & minister & india \\
prices & dollar & general & jennings & elections & political \\
investors & markets & week & monica & voting & islamic \\

\hline
\multicolumn{2}{c||}{Topic 4} &
\multicolumn{2}{c||}{Topic 5} &
\multicolumn{2}{c}{Topic 6}\\
\hline
\hline

lewinsky & lewinsky & game & team & military & cable \\
house & starr & team & hockey & indonesia & military \\
white & house & play & game & students & italy \\
grand & white & goal & players & jakarta & italian \\
jury & monica & players & canada & suharto & marine \\
starr & grand & period & play & indonesian & plane \\
president & jury & end & teams & police & car \\
monica & independent & scored & goal & anti & accident \\
clinton & counsel & coach & olympic & troops & officials \\
investigation & investigation & round & czech & protests & ski \\

\hline
\multicolumn{2}{c||}{Topic 7} &
\multicolumn{2}{c||}{Topic 8} &
\multicolumn{2}{c}{Topic 9}\\
\hline
\hline

\newpage

united & united & suharto & suharto & rudolph & gold \\
states & states & algeria & indonesia & clinic & won \\
japan & nations & algerian & president & bombing & medal \\
nations & american & algiers & jakarta & downhill & olympic \\
japanese & countries & islamic & indonesian & race & race \\
countries & britain & killed & economic & cup & mens \\
china & japan & suhartos & indonesias & eric & cup \\
asia & americans & militants & political & birmingham & womens \\
london & russia & economic & imf & alabama & slalom \\
american & canada & european & students & truck & olympics \\

\hline
\multicolumn{2}{c||}{Topic 10} &
\multicolumn{2}{c||}{Topic 11} &
\multicolumn{2}{c}{Topic 12}\\
\hline
\hline

israel & israel & workers & workers & imf & economic \\
netanyahu & netanyahu & plant & general & economic & crisis \\
israeli & israeli & plants & motors & bank & asian \\
palestinian & palestinian & parts & strike & asian & financial \\
peace & peace & auto & plants & economy & asia \\
arafat & minister & flint & flint & financial & economy \\
palestinians & prime & motors & strikes & dlrs & billion \\
talks & talks & jobs & parts & crisis & countries \\
blair & arafat & strike & michigan & billion & south \\
london & palestinians & strikes & plant & government & japan \\

\hline
\multicolumn{2}{c||}{Topic 13} &
\multicolumn{2}{c||}{Topic 14} &
\multicolumn{2}{c}{Topic 15}\\
\hline
\hline

council & annan & court & super & pope & pope \\
weapons & council & supreme & bowl & cuba & cuba \\
annan & security & bowl & game & cuban & cuban \\
security & general & super & denver & castro & visit \\
iraq & iraq & bulls & broncos & visit & john \\
agreement & agreement & chicago & packers & paul & paul \\
inspectors & kofi & jazz & green & john & castro \\
resolution & secretary & utah & bay & havana & havana \\
butler & deal & jordan & san & mass & fidel \\
ambassador & baghdad & game & football & church & church \\

\hline
\multicolumn{2}{c||}{Topic 16} &
\multicolumn{2}{c||}{Topic 17} &
\multicolumn{2}{c}{Topic 18}\\
\hline
\hline

means & american & olympic & nagano & spkr & spkr \\
order & going & nagano & games & starr & voice \\
terms & public & hockey & japan & peter & news \\
view & news & winter & olympic & voice & correspondent \\
certainly & york & gold & olympics & jennings & peter \\
reason & city & olympics & winter & correspondent & announcer \\
interview & political & team & japanese & abc & jennings \\
change & viagra & medal & athletes & camera & abc \\
simply & end & games & sports & tonight & headline \\
actually & country & athletes & japans & rush & voa \\

\hline
\multicolumn{2}{c||}{Topic 19} &
\multicolumn{2}{c||}{Topic 20} &
\multicolumn{2}{c}{Topic 21}\\
\hline
\hline

\newpage

mckinney & mckinney & iraq & iraq & winfrey & winfrey \\
kaczynski & court & gulf & military & government & show \\
sergeant & sexual & persian & gulf & show & texas \\
martial & major & aircraft & war & claims & beef \\
gene & sergeant & cable & attack & beef & oprah \\
mckinneys & gene & marine & strike & ruled & cattle \\
enlisted & army & italian & force & ranchers & ranchers \\
obstruction & martial & military & action & oprah & prices \\
armys & enlisted & accident & persian & disease & disease \\
soldier & misconduct & italy & kuwait & cattle & cow \\

\hline
\multicolumn{2}{c||}{Topic 22} &
\multicolumn{2}{c||}{Topic 23} &
\multicolumn{2}{c}{Topic 24}\\
\hline
\hline

nuclear & nuclear & tobacco & tobacco & economic & says \\
pakistan & india & bill & industry & cohen & headline \\
india & pakistan & industry & bill & defense & voa \\
tests & tests & senate & companies & asia & announcer \\
test & test & smoking & smoking & secretary & correspondent \\
pakistani & pakistani & companies & legislation & william & hes \\
sanctions & indias & legislation & congress & asian & attorney \\
indias & indian & congress & billion & crisis & king \\
testing & testing & republicans & settlement & financial & ginsburg \\
indian & weapons & billion & settlement & says & reports \\

\hline
\multicolumn{2}{c||}{Topic 25} &
\multicolumn{2}{c||}{Topic 26} &
\multicolumn{2}{c}{Topic 27}\\
\hline
\hline

nagano & bombing & school & school & reporter & reporter \\
house & clinic & students & students & king & cnn \\
white & rudolph & jonesboro & arkansas & ray & going \\
tucker & birmingham & boys & jonesboro & cnn & news \\
texas & federal & arkansas & middle & hes & headline \\
death & alabama & police & shooting & martin & reports \\
mrs & womens & shooting & boys & james & martin \\
games & eric & children & police & family & tell \\
ceremony & killed & middle & fire & going & sources \\
execution & atlanta & teacher & teacher & luther & king \\

\hline
\multicolumn{2}{c||}{Topic 28} &
\multicolumn{2}{c||}{Topic 29} &
\multicolumn{2}{c}{Topic 30}\\
\hline
\hline
\caption{Ten most relevant words for the remaining of the extracted topics by the FroNMF and wL1-NMF ($\lambda=0.08$) on the TDT2 dataset.}
\label{tab:topic_mod_qual_sup}
\end{longtable}

\section*{Supplementary to Imaging Mass Spectrometry (IMS) application}
\revise{This section provides the analysis of the effect of the $\lambda$ parameter in the wL1-NMF model applied to the pig206 dataset with injected artifacts. To perform the analysis we choose the same matrix used to obtain the results proposed in Figure and Figure. In particular, how the choice of the parameter affects the quality of the final reconstruction, the robustness to the outliers, and the interpretability of the $W$ factor.}

\begin{figure}[h!]
    \centering
    \includegraphics[width=\linewidth]{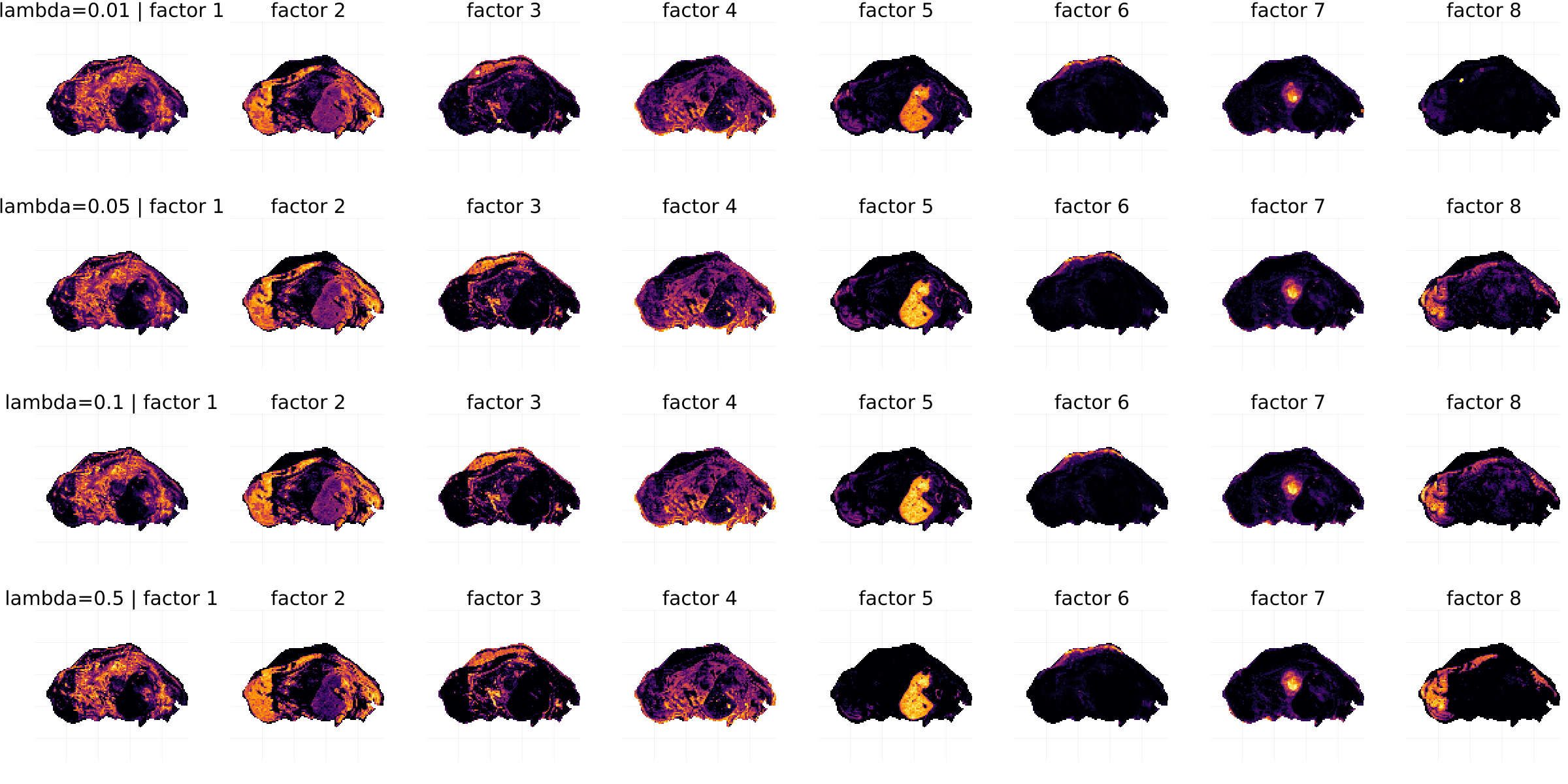}
    \caption{\revise{Spatial abundance maps obtained from the columns of the factor $W$ of rank-8 wL1-NMF with four different regularization parameter $\lambda$ on the corrupted pig206 dataset. Interpretation of the factors: general background/tissue (factor 1), mixed brain/liver (factor 2), subcutaneous adipose layer (factor 3), connective tissue (factor 4), liver (factor 5), skin (factor 6), heart (factor 7), brain and spinal cord (factor 8). }}
    \label{fig:IMS_factor_W_supp}
\end{figure}
\begin{table}[h!]
    \centering
    \begin{tabular}{ccccc}
        \hline
        $\lambda$ & L1 error & Fro error & $\rho_{\max}$ & $\rho_{\text{mean}}$ \\
        \hline
        $0.01$ & $0.877 $ & $2.58 $& $ 0.344 $ &  0.102  \\
        \hline
       $0.05$& $0.507$ & $0.727 $ & $0.234$ &  $0.065$\\
        \hline
         $0.1$& $\mathbf{0.506} $ & $0.527$ & $0.203 $ &  $\mathbf{0.062}$\\
        \hline
         $0.5$& $0.582 $ & $\mathbf{0.0478} $ & $\mathbf{0.202}$ &  $0.066$\\
        \hline
    \end{tabular}
    \caption{\revise{Comparison of wL1-NMF for different values of $\lambda$ applied to pig206 dataset with high intensity injected outliers. The results are averaged over 10 runs and show two metric for assessing the quality of the reconstruction (L1 and Fro error) and the max and mean correlation ($\rho_{\max}$ and $\rho_{\text{mean}}$) between NMF factors and binary signal of the outliers. }}
    \label{tab:comp_pig_supp}
\end{table}

\revise{These results show that $\lambda=0.1$ and $\lambda=0.5$ are reasonable choice for the regularization parameter. Moreover, the optimal choice depends on the primary goal of the NMF model. For instance, if the final target is obtaining an accurate low-rank representation of the data in term of relative error 
$\lambda=0.1$ achieves the lowest L1 error while $\lambda=0.5$ the lowest Fro error, as shown in Table~\ref{tab:topic_mod_qual_sup}. Nevertheless, if more importance is given to the isolation of dominant tissues, such as liver and brain in factors 5 and 7, respectively
$\lambda=0.5$ yields a more refined separation by removing background structures. On the other hand, $\lambda=0.01$ yields a model which is not sufficiently robust to the outlier noise since factor 8 catches the outlier signal.}

\end{document}